\documentclass[Afour,sageh,times,sageh]{sagej}

\usepackage{moreverb,url}

\usepackage[colorlinks,bookmarksopen,bookmarksnumbered,citecolor=red,urlcolor=red]{hyperref}

\newcommand\BibTeX{{\rmfamily B\kern-.05em \textsc{i\kern-.025em b}\kern-.08em
T\kern-.1667em\lower.7ex\hbox{E}\kern-.125emX}}

\def\volumeyear{2019}

\usepackage{graphicx} 
\usepackage{times} 
\usepackage{amsmath} 
\usepackage{amssymb}  
\usepackage{subfigure}
\usepackage{makecell}
\usepackage{siunitx}
\usepackage{algorithm}
\usepackage[noend]{algpseudocode}
\usepackage{array}
\usepackage{tabularx}
\usepackage{adjustbox}
\usepackage{makecell}
\usepackage{booktabs}
\usepackage{bbm}

\newtheorem{theorem}{Theorem}
\newtheorem{lemma}{Lemma}

\newtheorem{remark}{Remark}

\DeclareMathOperator\erfc{erfc}

\usepackage{color}
\newcommand{\ie}{{\em i.e.}}
\newcommand{\eg}{{\em e.g.}}

\begin{document}

\title{FSMI: Fast computation of Shannon Mutual Information for information-theoretic mapping}

\author{Zhengdong Zhang\affilnum{1}, Trevor Henderson\affilnum{1},  Sertac Karaman\affilnum{2}, Vivienne Sze\affilnum{1}}

\affiliation{\affilnum{1}Department of Electrical Engineering and Computer Science, MIT, Cambridge, MA, USA\\
\affilnum{2}Department of Aeronautics and Astronautics, MIT, Cambridge, MA, USA}

\corrauth{Zhengdong Zhang, 
Massachusetts Institute of Technology,
Cambridge, MA, US}

\email{zhangzd@mit.edu}


\begin{abstract}

Exploration tasks are embedded in many robotics applications, such as search and rescue and space exploration. 
Information-based exploration algorithms aim to find the most informative trajectories by maximizing an information-theoretic metric, such as the mutual information between the map and potential future measurements. 
Unfortunately, most existing information-based exploration algorithms are plagued by the computational difficulty of evaluating the Shannon mutual information metric. 
In this paper, we consider the fundamental problem of evaluating Shannon mutual information between the map and a range measurement. 
First, we consider 2D environments. We propose a novel algorithm, called the Fast Shannon Mutual Information (FSMI). The key insight behind the algorithm is that a certain integral can be computed analytically, leading to substantial computational savings.  
Second, we consider 3D environments, represented by efficient data structures, \eg, an OctoMap, such that the measurements are compressed by Run-Length Encoding (RLE). We propose a novel algorithm, called FSMI-RLE, that efficiently evaluates the Shannon mutual information when the measurements are compressed using RLE.
For both the FSMI and the FSMI-RLE, we also propose variants that make different assumptions on the sensor noise distribution for the purpose of further computational savings. 
We evaluate the proposed algorithms in extensive experiments. In particular, we show that the proposed algorithms outperform existing algorithms that compute Shannon mutual information as well as other algorithms that compute the Cauchy-Schwarz Quadratic mutual information (CSQMI). In addition, we demonstrate the computation of Shannon mutual information on a 3D map for the first time.

 \end{abstract}
\maketitle


\section{Introduction}

\begin{table*}
    \centering
    \begin{tabular}{|c|c|c|c|c|}
    \hline
\makecell{\bf Time complexity of\\ \bf proposed algorithms}  & \makecell{No encoding of \\ measurement (2D)} & \makecell{Time\\Complexity} & \makecell{Run-length encoding \\ of measurement (3D) } & \makecell{Time\\Complexity}  \\\hline
         \makecell{Exact method assuming \\Gaussian sensor noise}  & \makecell{{\tt FSMI}\\{\small\em (Section~\ref{sec:fsmi_algorithm})}} & $O(n^2)$  & \makecell{{\tt FSMI-RLE}\\{\small\em (Section~\ref{sec:exact_fsmi_rle})}} & $O(n_r^2)$ \\ \hline
         \makecell{Approximate method via\\ Gaussian truncation} & \makecell{ {\tt\ Approx-FSMI} \\ {\small\em (Section~\ref{sec:fsmi_approx})} }  & $O(n\,\Delta)$ & \makecell{ {\tt Approx-FSMI-RLE} \\ {\small\em (Section~\ref{sec:fsmi_rle_approx})} } &  $O(n_r \,\Delta)$ \\ \hline
         \makecell{Exact method assuming\\ uniform sensor noise} & \makecell{{\tt Uniform-FSMI} \\ {\small\em (Section~\ref{sec:fsmi_uniform})}} & $O(n)$ & \makecell{ {\tt Uniform-FSMI-RLE} \\ {\small\em (Section~\ref{sec:fsmi_rle_uniform})} } &  $O(n_r \, H)$ \\ \hline
    \end{tabular}
    \vspace{0.1in}
    
    \caption{
    The time complexity of all six algorithms proposed in this paper are presented in the table, where $n$ is the number of cells that the measurement intersects, $\Delta$ is the truncation length of Gaussian truncation, $n_r$ is the number of classes of cells in a run-length-encoding compression of the measurement, and $H$ is the width of the uniform distribution as measured by the number of cells that the uniform distribution intersects.
    The algorithms are classified depending on two factors: {\em (i)} their assumptions of sensor noise, {\em (ii)} their encoding of measurements. The {\tt FSMI} algorithm is an exact method that assumes Gaussian sensor noise, similar to existing algorithms in the literature. The {\tt Approx-FSMI} algorithm is an approximate method that approximates Guassian sensor noise by truncating the Gaussian distribution. The {\tt Uniform-FSMI} algorithm approximates by assuming uniform sensor noise. 
    The {\tt FSMI-RLE}, {\tt Approx-FSMI-RLE}, and {\tt Uniform-FSMI-RLE} are their versions that represent measurements in run-length encoding, which is essential when working with large-scale maps for three-dimensional environments. 
    The table shows how various approximations (going down in the table) and run-length encoding (going right in the table) reduce time complexity. 
    }

    
    \label{tab:algorithm_summary}
\end{table*}

Robot exploration tasks are embedded and essential in several applications of robotics, including disaster response and space exploration. 
%
%
The problem has received a large amount of attention over the past few decades, resulting in a rich literature.

On the one hand, {\em geometry-based frontier exploration algorithms} approach this problem with heuristics that typically navigate the robot to the frontier of the well known portion of the environment~\citep{yamauchi1997frontier}.
Researchers investigated various objective functions
\citep{burgard2005coordinated, gonzalez2002navigation}, and \cite{holz2011comparative} surveys their performance in practical scenarios.  
These heuristics are very efficient from a computational point of view. However, they lack any rigorous reasoning about information, which makes them relatively inefficient in terms of the path spanned by the robot while exploring the environment~\citep{Elfes1996robot, cassandra1996acting, moorehead2001autonomous, bourgault2002information}. In addition, it is hard to extend the geometry that they rely on to three-dimensional environments~\citep{shen2012stochastic}.

On the other hand, {\em information-based mapping and exploration} techniques consider paths that aim to maximize principled information-theoretic metrics to actively maximize the information collected by the robot. %
The Shannon mutual information between a perspective scan and the occupancy grid is used for exploration in~\cite{bourgault2002information}.  \cite{visser2008balancing} introduces an objective function that balances the mutual information and the moving cost. \cite{charrow2014approximate} developed an approximated mutual information representation to efficient multi-robot control. This mutual information metric is also widely used in many other related applications. \cite{kollar2008efficient} proposes an information theoretic objective for SLAM. \cite{marchant2014bayesian} uses it to perform continuous path planning.
\cite{julian2014mutual} established a rigorous theory and algorithms for evaluating Shannon mutual information between a measurement and the map.
While information-based mapping algorithms using Shannon mutual information (MI) provide guarantees on the exploration of the environment, the evaluation of Shannon MI, {\em e.g.}, by the algorithm provided by~\cite{julian2014mutual}, is computational demanding. The run time of the algorithm scales quadratically with the spatial resolution of the occupancy grid and linearly with the numerical integration resolution of the range measurement due to the absence of an analytical solution. 
It has been pointed out that the speed at which mutual information is evaluated can limit the planning frequency, which in turn limits the velocity of the robot and the exploration speed of the environment~\citep{nelson2015information}.

Towards designing algorithms that are computationally more efficient,~\cite{charrow2015csqmi} proposed the use of an alternative information metric, called Cauchy-Schwarz Quadratic Mutual Information (CSQMI). They show that the integrations in CSQMI can be computed analytically. Additionally, they show a close approximation of CSQMI can be evaluated in time that scales linearly with respect to the spatial resolution of the occupancy grid. It is reported by~\cite{charrow2015csqmi} that CSQMI can be computed substantially faster than Shannon MI, and it behaves similarly to Shannon MI in experiments. Several other works adopted CSQMI as the information metric for exploration.
%
For instance,~\cite{charrow2015information} propose a hybrid method with global and local trajectory optimization based on CSQMI. \cite{nelson2015information} propose an adaptive occupancy grid compression algorithm and uses CSQMI to design the planner on the compressed map. \cite{Tabib2016Computationally} combine CSQMI with trajectory optimization and compression to build an energy-efficient cave exploration system with a drone.



In this paper, we focus on the fundamental problem of computing the Shannon mutual information metric between a range measurement and the map. 
We propose a new class of algorithms for efficient evaluation of this metric. 

First, we propose the Fast Shannon Mutual Information method, also called the {\tt FSMI} algorithm, which evaluates Shannon mutual information exactly for Gaussian distributed sensor noise characteristics as commonly assumed in the literature. The key idea behind the {\tt FSMI} algorithm is to analytically evaluate a certain integral, which leads to substantial computational savings. Second, we propose two variants of this algorithm, called the {\tt Approx-FSMI} and {\tt Uniform-FSMI}, which approximate the evaluation of the same metric by truncating the Gaussian distribution (which was first introduced by~\cite{charrow2015csqmi}) and by assuming a uniform sensor noise distribution, respectively. Third, we propose the variants of these algorithms, called {\tt FSMI-RLE}, {\tt Approx-FSMI-RLE}, and {\tt Uniform-FSMI-RLE}, that can handle measurements represented in run-length encoding. The run-length encoding is a certain kind of compression, which is particularly efficient for working with three-dimensional maps. 
The time complexity of the proposed algorithms are presented in Table~\ref{tab:algorithm_summary}.
For instance, the time complexity of the {\tt FSMI} algorithm is $O(n^2)$, where $n$ is the number of cells in the map that the range measurement intersects. 
In contrast, the time complexity of the existing algorithm for computing Shannon mutual information is $O(n^2 \, \lambda_z)$, where $\lambda_z$ is resolution for a certain numerical integral. 
The {\tt FSMI} algorithm avoids this numerical integral, and various approximations of sensor noise characteristics and encodings of the measurement vector provide even more efficiency. 

We demonstrate this theoretical computational efficiency in experiments. In particular, we present a comparison of various methods for computing mutual information in a computational study, in simulations, and in experiments involving a ground robot equipped with a planar laser scanner. In addition, we present mapping in a three-dimensional environment involving a ground robot equipped with a three-dimensional Velodyne VLP-16 laser scanner.

\section{Preliminaries}
\label{sec:preliminary}

This section is devoted to the introduction of preliminaries and our notation, which we use throughout the paper.
We briefly review the occupancy grid in Section~\ref{sec:preliminary:map}, the Shannon mutual information metric and its computation in Section~\ref{sec:preliminary:shannon}, and the Cauchy-Schwarz quadratic mutual information metric and its computation in Section~\ref{sec:preliminary:csqmi}. 
Finally, in Section~\ref{sec:preliminary:problem}, we formulate the single-beam mutual-information evaluation problem considered in this paper. 

\subsection{The Occupancy Grid}\label{sec:preliminary:map}

We use the occupancy grid to model the environment. 
Following standard convention, we assume that all the occupancy cells are independent and that a Bayesian filter is used to update the occupancy probabilities.

The occupancy grid is denoted by the random variables $M = \{M_1, \dots, M_{K}\}$, where $K$ is the number of cells and $M_i\in\{0, 1\}$ is the binary random variable that indicates the occupancy of $i$-th cell. In this case, $M_i = 0$ indicates an empty cell, $M_i = 1$ indicates an occupied cell. The realization of the random variable $M_i$ is denoted by $m_i$.

The robot is equipped with a range measurement sensor. 
Let the perspective range measurement be denoted by the random variable $Z$. The realization of the random variable $Z$ is denoted by $z$. 
The perspective range measurement at time $t$ is denoted by $Z_t$, and its realization is denoted by $z_t$. The measurements obtained up to time $t$ is denoted by $z_{1:t}$.
Typically, the robot acquires multiple range measurements at the same time, \eg, measuring range in various directions. The sequence of range measurements obtained at time $t$ is denoted by $z_t = (z_t^1,\ldots,z_t^{n_z})$ where $n_z$ is the number of beams in a scan.
Unless explicitly stated otherwise, we assume that the noise distribution of the sensor is a Gaussian with standard deviation $\sigma$ regardless of the travel distance of the beam.

The state of the robot is denoted by $x$. In this paper, the state variable is the pose of the robot. 
We denote the state at time $t$ by $x_t$. We denote the sequence of states from the initial time through time $t$ by $x_{1:t}$. 
The measurements obtained by the robot are a stochastic function of the map, the sensor model, and the state of the robot at that time as shown in Figure~\ref{fig:preliminaries:occupancy_map}.

\begin{figure}
    \centering
    \includegraphics[width=0.9\columnwidth]{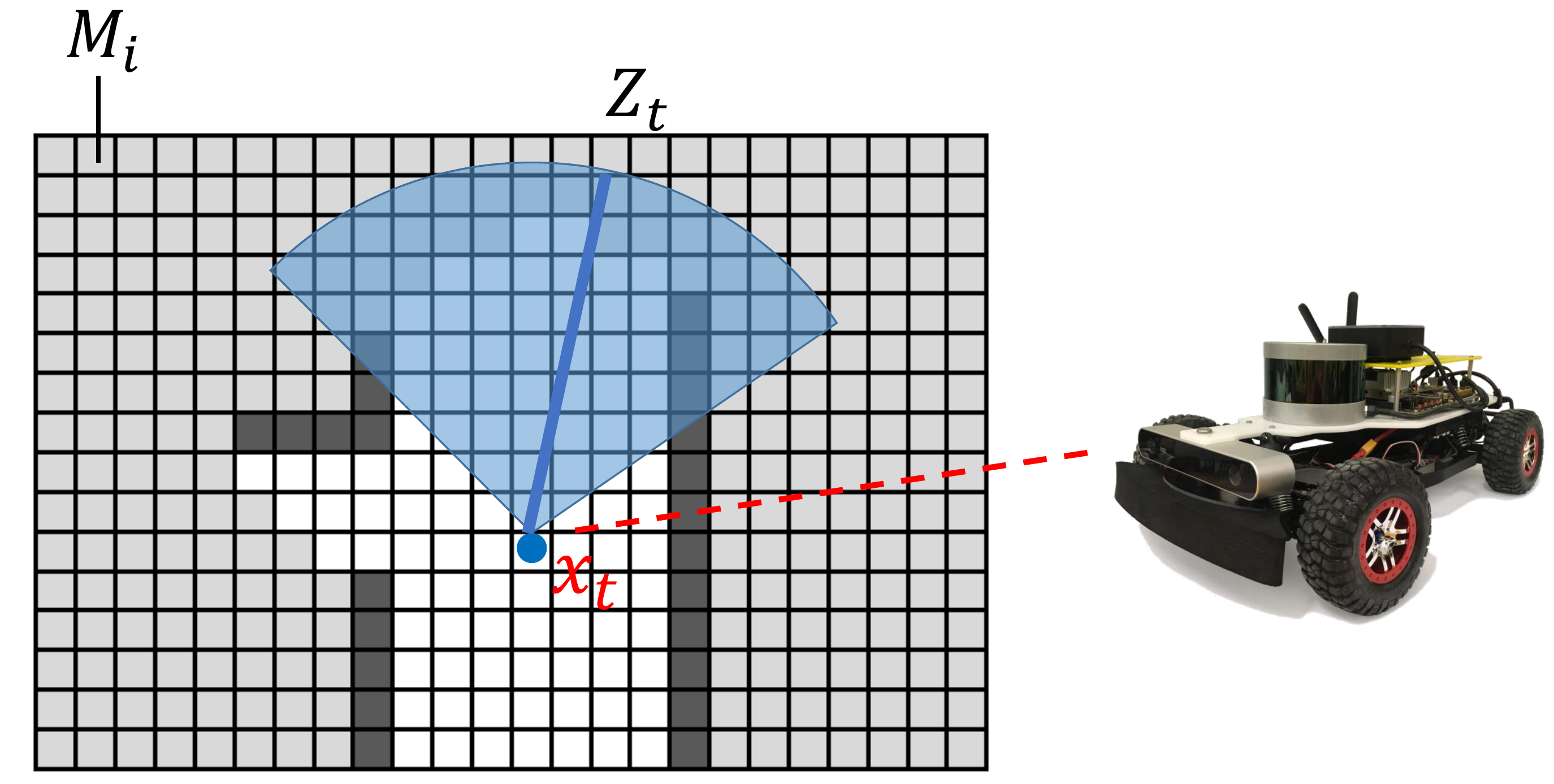}
    \caption{A 2D environment is represented by an occupancy grid. Each cell in the map is associated with a binary random variable $M_i$ indicating the occupancy of the cell. At time $t$, a vehicle can scan at location $x_t$ with perspective range measurement being $Z_t$.}
    \label{fig:preliminaries:occupancy_map}
\end{figure}

We make the the standard independence assumption among the occupancy grid cells, \ie, 
\begin{align*}
&P(M_1 = m_1, \ldots, M_K=m_k|z_{1:t}, x_{1:t}) \\ &\qquad = \prod_{1\le i\le K} P(M_i = m_i|z_{1:t}, x_{1:t}),
\end{align*}
where $P(\cdot)$ is the probability function. We assume that the robot has no prior information on the environment, that is, $P(M_i=1)=P(M_i=0)=0.5$ for all $M_i \in M$. Once a measurement is obtained, the standard Bayesian filter can be used to update the occupancy grid according to
\begin{equation}
\begin{split}
\label{eqn:bayesian_filter}
\frac{P(M_i = 1|z_{1:t}, x_{1:t})}{P(M_i = 0|z_{1:t}, x_{1:t})} = &\frac{P(M_i = 1|z_{1:(t - 1)}, x_{1:(t - 1)})}{P(M_i  = 0|z_{1:(t - 1)}, x_{1:(t - 1)})} \\
& \frac{P(M_i = 1|z_t, x_t)}{P(M_i = 0|z_t, x_t)}. 
\end{split}
\end{equation}

We denote the probability of occupancy of the $i$-th cell by 
$$o_i = P(M_i = 1 | z_{1:t}, x_{1:t}).$$
Additionally, we denote the odds ratio of a cell by
$$r_i = o_i / (1 - o_i).$$
The Bayesian filter given by Equation~\eqref{eqn:bayesian_filter} essentially updates $r_i$ for the cells related to the acquired range measurements.

\subsection{The Shannon Mutual Information Metric} \label{sec:preliminary:shannon}

The Shannon mutual information metric between the map $M$ and the measurement $Z$ is defined as follows:
\begin{equation}
\begin{split}
I(M;Z) = \sum_{m\in\{0, 1\}^K } \int_{z\ge 0} P(Z=z, M = m)&\\
\log{\frac{P(Z=z, M = m)}{P(Z=z)P(M = m)}} dz.&
\end{split}
\end{equation}

In this section, we review the algorithm proposed by~\cite{julian2014mutual} that computes the Shannon mutual information for single range measurement, such as a LiDAR beam. Throughout the paper, we use the word beam and the phrase range measurement interchangeably to describe this measurement.
For notational simplicity, we omit the conditional probability terms $x_{1:t}$ and $z_{1:t}$. Moreover, we use $M'=(M_1, ..., M_{n})$ to represent the cells that this single range measurement intersects. Note that a beam tells no information about the cells in $M$ that it does not intersect. Therefore, $I(M'; Z) = I(M; Z)$. We further assume that cells in $M'$ are listed in ascending order by their distance from the sensor.
In addition, let $\delta_i(z)$ approximate the odds ratio inverse sensing model~\citep{thrun2005probabilistic} for the cell $M_i$:
\begin{equation}
\label{eqn:inverse_sensing}
\delta_i(z) = \left\{\begin{array}{ll}\delta_{occ} & \textrm{$z$ indicates $M_i$ is occupied} \\ \delta_{emp} & \textrm{$z$ indicates $M_i$ is empty} \\ 1 & \textrm{otherwise}\end{array}\right.,
\end{equation}
where $\delta_{occ} > 1$ and $\delta_{emp} < 1$ are hyper parameters. As shown by~\cite{julian2014mutual}, The mutual information between the beam and a single cell $M_i$ is
\begin{equation}
\label{eqn:mi_original}
I(M_i; Z) = \int_{z\ge 0} P(Z = z)f(\delta_i(z), r_i) dz,
\end{equation}
where $P(Z=z)$ is the measurement prior
\begin{align}
\label{eqn:measurement_prior}
\begin{split}
&P(Z = z) \\& \qquad= P(e_0)P(Z = z|e_0) + \sum_{j=1}^n P(e_j) P(Z = z|e_j).
\end{split}
\end{align}
and $f(\delta_i(z), r_i)$ is the following function:
\begin{equation}
    f(\delta, r) = \log\left(\frac{r+1}{r+\delta^{-1}}\right) - \frac{\log\delta}{r\delta+1}.
\end{equation}
Here with a slight abuse of notation, we define $P(e_j)$ to denote the probability that the $j$-th cell is the first occupied cell on the beam and all the cells before the $j$-th cell are empty. Similarly, we let $P(e_0)$ represent the probability that all cells are empty:
\begin{equation}
\label{eqn:Pej}
    \begin{split}
P(e_0) & = \prod_{l = 1}^n (1 - o_l); \\
P(e_j) & = o_j\prod_{l < j}(1-o_l), \qquad \mbox {for all }j \ge 1. 
    \end{split}
\end{equation}

The function $P(Z = z|e_j)$, as a function of $z$, denotes the probability distribution of the range measurement, if the beam passes through the occupancy cells before the $j$-th cell and is blocked by the $j$-th cell. It is determined by the distance between the $j$-th cell and the sensor as well as the previously discussed sensor noise model. For notational simplicity, we abbreviate $P(Z = z|e_j)$ as $P(z|e_j)$ in the discussion that follows, particularly in Section~\ref{sec:fsmi_correctness} where we prove the correctness of the main result.

Since Equation~\eqref{eqn:mi_original} does not have a known analytical solution, \cite{julian2014mutual} evaluate it numerically by discretizing $z$:
\begin{equation}
I(M_i; Z) = \sum_z P(Z = z)f(\delta_i(z), r_i) \lambda_z^{-1},
\end{equation}
where $\lambda_z$ is the resolution for numerical integration.
The mutual information then is computed as $\sum_{i = 1}^n I(M_i; Z)$. The time complexity of this algorithm is $O(n^2\lambda_z)$~\citep{julian2014mutual}.

\subsection{The Cauchy-Schwarz Quadratic Mutual Information (CSQMI) Metric} \label{sec:preliminary:csqmi}
Let $P(Z, M)$ be the joint probability distribution of $Z$ and $M$, and $P(Z), P(M)$ be the probability distribution of $Z$ and $M$ respectively. The CSQMI~\citep{principe2010information, charrow2015csqmi} between $M$ and $Z$ is defined as
\begin{equation}
I_{CS}(M; Z) = D_{CS}\left(P(Z, M), P(Z)P(M)\right),
\end{equation}
where $D_{CS}(\cdot, \cdot)$ is the Cauchy-Schwarz divergence between two probability distributions $P_1$ and $P_2$, defined as follows:
\begin{equation}
\begin{split}
    &D_{CS}(P_1(X), P_2(X))=\\&-\frac{1}{2}\log\frac{\left(\int_x P_1(X=x)P_2(X=x) dx\right)^2}
    {\left(\int_x P_1(X=x)^2 dx\right)\left(\int_x P_2(X=x)^2 dx\right)}.
\end{split}
\end{equation}

Both the Shannon mutual information and CSQMI measure the difference between $P(M, Z)$ and $P(M)P(Z)$; \cite{charrow2015csqmi} show that the mutual information map and the CSQMI map of the same occupancy grid look visually similar in practical robotics applications. 

\cite{charrow2015csqmi} has shown that CSQMI for a beam that intersects with $n$ cells can be evaluated analytically in $O(n^2)$ time. This is lower than the $O(n^2\lambda_z)$ complexity of Shannon mutual information, because numerical integration is avoided in CSQMI. In this paper, FSMI is compared against CSQMI. For completeness, we reiterate the formula to compute CSQMI by~\citep{charrow2015csqmi}:
\begin{align}
\nonumber I_{CS}(M; Z) &=  \log\sum_{l=0}^Cw_l\mathcal{N}(0, 2\sigma^2)\\
\nonumber& \hspace{-0.3in} +\log \Big(\prod_{i=1}^{n}(o_i^2 + (1 - o_i)^2)\\ 
\nonumber& \qquad \sum_{j=0}^{n}\sum_{l=0}^{n} P(e_j)P(e_l)\mathcal{N}(\mu_l - \mu_j, 2\sigma^2)\Big)\\
&\hspace{-0.5in} -2\log{\sum_{j=0}^{n}\sum_{l=0}^{n}P(e_j)w_l \mathcal{N}(\mu_l - \mu_j, 2\sigma^2)},
\label{eqn:csqmi_exact}
\end{align}
where $\mathcal{N}(x, \sigma^2)$ is the probability density function at $x$ of a normal distribution of zero mean and standard derivation of $\sigma$, $\mu_l$ is the distance from the center of the $l$-th cell to the range sensor on the robot and 
$$
w_l = P(e_l)^2\prod_{j < l}(o_j^2 + (1-o_j)^2).
$$

\cite{charrow2015csqmi} proposed a close approximation to CSQMI that truncates the tails of a Gaussian distribution. This enables CSQMI to be computed approximately in $O(n)$ time. Specifically, each double sum in Equation~\eqref{eqn:csqmi_exact} can be approximated as follows:
\begin{equation}
\sum_{j=0}^{n}\sum_{l=j-\Delta}^{j+\Delta} \alpha_{j, l}\mathcal{N}(\mu_l - \mu_j, 2\sigma^2),
\end{equation}
where $\alpha_{j, l}$ represents the corresponding coefficient in the double sum, and $\Delta$ is a small constant such as $\Delta = 3$.

\subsection{Problem Formulation}\label{sec:preliminary:problem}
A typical information-theoretic exploration strategy is to generate a set of potential trajectories, evaluate mutual information along each of trajectory, and choose the one with the highest mutual information per travel cost~\citep{nelson2015information,charrow2015csqmi,charrow2015information,Tabib2016Computationally}.

In order to evaluate the mutual information along a trajectory, the trajectory is typically discretized in the state space, and mutual information is computed at each state and then summed, as shown in Figure~\ref{fig:preliminaries:path_sample}. To avoid double counting, 
cells that have already contributed to the total mutual information are marked and not used in future computations~\citep{charrow2015csqmi}.
\begin{figure}
    \centering
    \includegraphics[width=0.7\columnwidth]{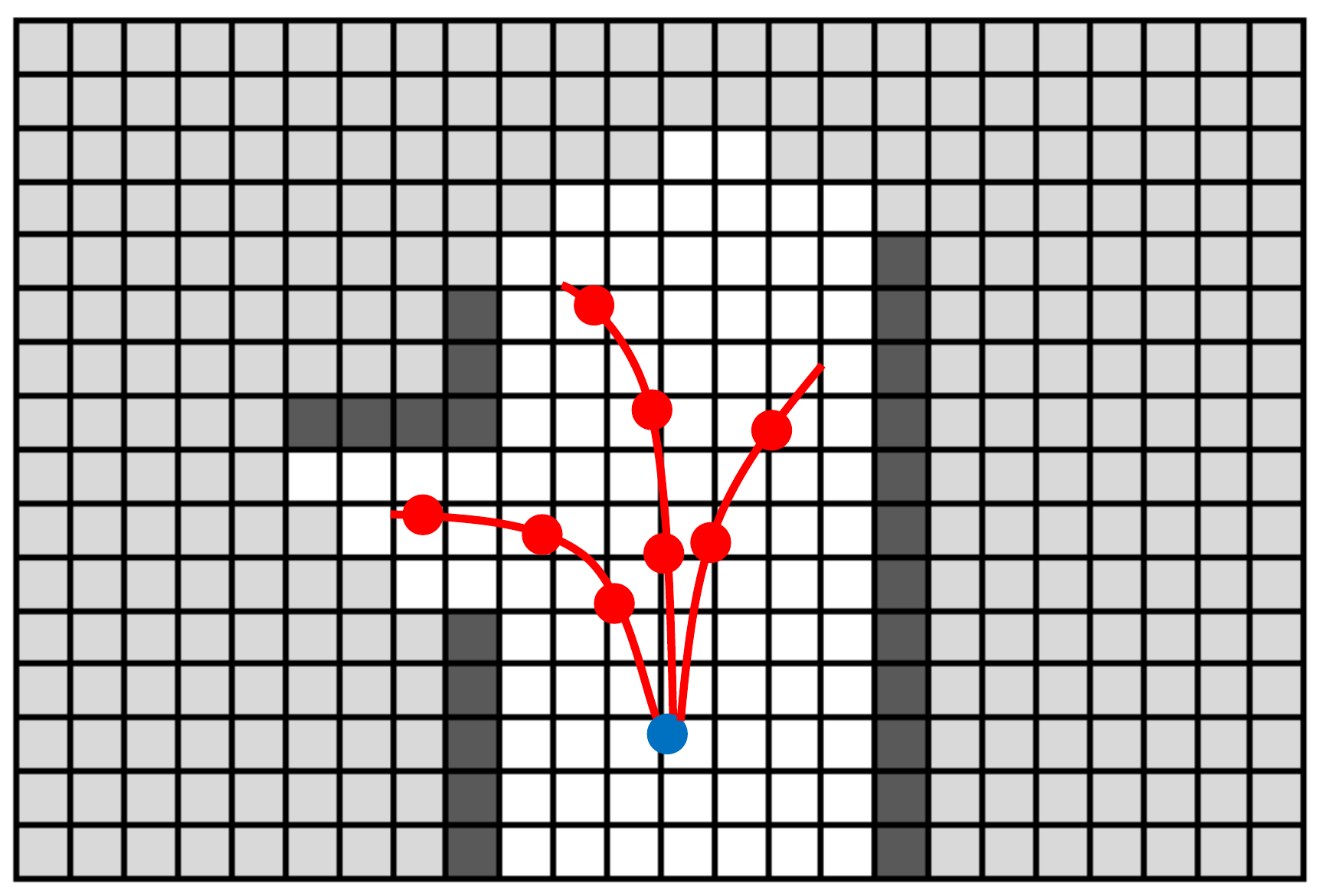}
    \caption{Each candidate trajectory (marked by red) is discretized into a set of states. Shannon mutual information is evaluated between the perspective scan measurements at each state and the map, then summed up along the trajectory. }
    \label{fig:preliminaries:path_sample}
\end{figure}
It has been shown in previous research~\citep{julian2014mutual} that the fundamental problem of information-based exploration using range sensing is to efficiently evaluate the mutual information between the map and single range measurement. 
This problem is solved several times in typical mutual-information-based exploration algorithms. We call this the {\em single-beam mutual-information computation} problem.

Consider a range measurement sensor, \eg, using a LiDAR beam emanating from the sensor.
Again, we let $M'=(M_1, \ldots, M_{n})$ denote the vector of binary random variables representing all the occupancy cells that the beam intersects. Let the random variable $Z$ denote the corresponding range measurement.
We wish to measure the dependence between these random variables using Shannon mutual information. The mutual information between the range measurement $Z$ and a single cell $M_i$ can be computed using Equation~\eqref{eqn:mi_original}. 
We also follow the standard assumption that the mutual information between the range measurement $Z$ and all of the cells in $M$ is the summation of the mutual information between the beam and each individual cell:
\begin{align}
\begin{split}
\label{eqn:mi_sum}
I(M'; Z) & = \sum_{i=1}^n I(M_i; Z)\\
&\hspace{-.5in} = \sum_{i=1}^n \int_{z\ge 0} P(Z = z)f(\delta_i(z), r_i) dz,\\
&\hspace{-.5in} = \sum_{i=1}^n \int_{z\ge 0} P(Z = z) \Big( \log\left(\frac{r+1}{r+\delta^{-1}}\right) - \frac{\log\delta}{r\delta+1} \Big) dz,
\end{split}
\end{align}
where $P(Z = z)$ is the measurement prior defined in Equation~\eqref{eqn:measurement_prior}.

This paper focuses on efficient algorithms and efficient approximations for computing the quantity defined in Equation~\eqref{eqn:mi_sum}. We first discuss the algorithms in 2D. Then, we propose an extension of the algorithms to the 3D mapping problem that further reduces the computation complexity using the structure of a 3D map.


\section{The Fast Shannon Mutual Information (FSMI) Algorithm for 2D Mapping}
\label{sec:fast_mi}

This section is devoted to the presentation of the FSMI algorithm for mapping in the 2D plane. 
The FSMI algorithm is presented in Section~\ref{sec:fsmi_algorithm}.
The correctness of the algorithm is proved in Section~\ref{sec:fsmi_correctness}, and its computational complexity is analyzed in Section~\ref{sec:fsmi_complexity}. Efficient, practical implementations of the FSMI enabled by tabulation and approximation techniques are discussed in Section~\ref{sec:fsmi_approx}. Finally, efficient implementations for the case when the sensor noise follows a uniform distribution, instead of a Gaussian distribution, are discussed in Section~\ref{sec:fsmi_uniform}. 

\subsection{The FSMI Algorithm}
\label{sec:fsmi_algorithm}
In this section, we present the Fast Shannon Mutual Information (FSMI) algorithm. The key idea behind the FSMI algorithm is the following: Instead of performing the summation  in Equation~\eqref{eqn:mi_sum} directly, FSMI computes $I(M; Z)$ (same as $I(M'; Z)$) holistically and analytically evaluates one of the resulting integrals, which leads to substantial computational savings.

Algorithm~\ref{alg:i(m, z)} summarizes this procedure with subroutines in Algorithms~\ref{alg:pei} and \ref{alg:c_k}. 
To describe the algorithms, let us first present our notation. 
Let $l_i$ be the distance from the beam's origin to the $i$-th cell where $l_1 = 0$ and $l_i$ is monotonically increasing, as shown in the middle of Figure~\ref{fig:fz_constant}.
Let us denote the center of the $i$-th cell by $\mu_i = (l_i + l_{i + 1}) / 2$. 
Just as in Equation~\eqref{eqn:Pej}, $P(e_j)$ is used to denote the probability that the $j$-th cell is the first non-empty cell in $M'$, \ie, 
\begin{align*}
P(e_j) = o_j\prod_{i < j} (1-o_i).
\end{align*} 
Let $P(e_0)$ be the probability that all cells are empty. $\Phi(\cdot)$ denotes the standard normal CDF.
We also define 
\begin{align}
\label{eqn:C_k}
C_k = f(\delta_{occ}, r_k) + \sum_{i < k} f(\delta_{emp}, r_i)
\end{align} 
where $\delta_{occ}$ and $\delta_{emp}$ are from the inverse sensor model defined in Equation~\eqref{eqn:inverse_sensing} and 
\begin{align}
\label{eqn:G_jk}
G_{j, k} = \int_{l_{k}}^{l_{k + 1}} P(z|e_j) dz
\end{align}.

In Algorithm~\ref{alg:i(m, z)}, the total mutual information is initialized to zero (Line \ref{line:alg_fsmi_I_to_0}). Then, Algorithm~\ref{alg:pei} and Algorithm~\ref{alg:c_k} are used to compute $P(e_j)$ for all $0\le j\le n$ (Line \ref{line:alg_fsmi_pej}) and $C_k$ for all $1\le k\le n$ (Line \ref{line:alg_fsmi_Ck}), respectively. Then, Algorithm~\ref{alg:i(m, z)} enumerates $j, k$ through $0$ to $n$, and for each pair $(k, j)$ it computes $G_{k, j}$ and accumulates $P(e_j)C_kG_{k, j}$ into the total mutual information (Lines \ref{line:alg_fsmi_for_start}-\ref{line:alg_fsmi_for_end}).

\begin{algorithm}[!tb]
\caption{The FSMI algorithm}\label{alg:i(m, z)}
\begin{algorithmic}[1]
\Require $\sigma$ and $l_i, o_i$ for $1\le i\le n$.
\State \label{line:alg_fsmi_I_to_0} $I\leftarrow 0$
\State \label{line:alg_fsmi_pej} Compute $P(e_j)$ for $0\leq j \leq n$ with Algorithm~\ref{alg:pei}.
\State \label{line:alg_fsmi_Ck} Compute $C_k$ for $1 \leq k \leq n$ with Algorithm~\ref{alg:c_k}.
\For{\label{line:alg_fsmi_for_start}$j=0$ \textbf{to} $n$}
    \For{$k = 0$ \textbf{to} $n$}
        \State $G_{k, j} \leftarrow \Phi((l_{k + 1} - \mu_j) / \sigma_j) - \Phi((l_{k} - \mu_j)/ \sigma_j)$
        \State\label{line:alg_fsmi_for_end} $I\leftarrow I + P(e_j)C_k G_{k, j}$
    \EndFor
\EndFor
\State \Return $I$
\end{algorithmic}
\end{algorithm}
\begin{algorithm}[!tb]
\caption{Evaluate $P(e_j)$ for $0 \leq j \leq n$}\label{alg:pei}
\begin{algorithmic}[1]
\Require $o_i$ for $1 \leq i \leq n$
\State $E_0\leftarrow 1$
\For{$j=1$ \textbf{to} $n$}
    \State $E_j \leftarrow E_{j - 1} (1 - o_j)$
    \State $P(e_j)\leftarrow E_{j - 1}o_j$
\EndFor
\State $P(e_0) = E_{n}$
\State \Return $P(e_j)$ for $0\le j\le n$
\end{algorithmic}
\end{algorithm}
\begin{algorithm}[!tb]
\caption{Evaluate $C_k$ for $1 \leq k \leq n$}\label{alg:c_k}
\begin{algorithmic}[1]
\Require $\delta_{emp}, \delta_{occ}$ and $r_i$ for $1 \leq i \leq n$
\State $q_0=0$
\For{$k=1$ \textbf{to} $n$}
    \State $q_k = q_{k - 1} + f(\delta_{emp}, r_k)$
    \State $C_k = q_{k - 1} + f(\delta_{occ}, r_k)$
\EndFor
\State \Return $C_k$ for $1\le k\le n$
\end{algorithmic}
\end{algorithm}

\subsection{Correctness of the FSMI algorithm}
\label{sec:fsmi_correctness}

The following theorem states the correctness of the FSMI algorithm, that is, the FSMI algorithm indeed returns $I(M';Z)$, \ie, the Shannon mutual information between the measurement $Z$ and $M'$, the part of the map the beam intersects with. 
\begin{theorem}[Correctness of FSMI]
\label{theorem:exact_shannon_mi}

The Shannon mutual information between the range measurement of the beam and all of the cells in $M'$ is
\begin{equation}
\label{eqn:fast_exact_mi}
I(M'; Z) = \sum_{j = 0}^{n}\sum_{k = 1}^{n} P(e_j)C_k G_{k, j}
\end{equation}
\end{theorem}

To prove this theorem, we need the following intermediate result regarding the structure of $f(\delta, r)$.

\begin{figure}[!tb]
\centering 
\includegraphics[width=0.85\columnwidth]{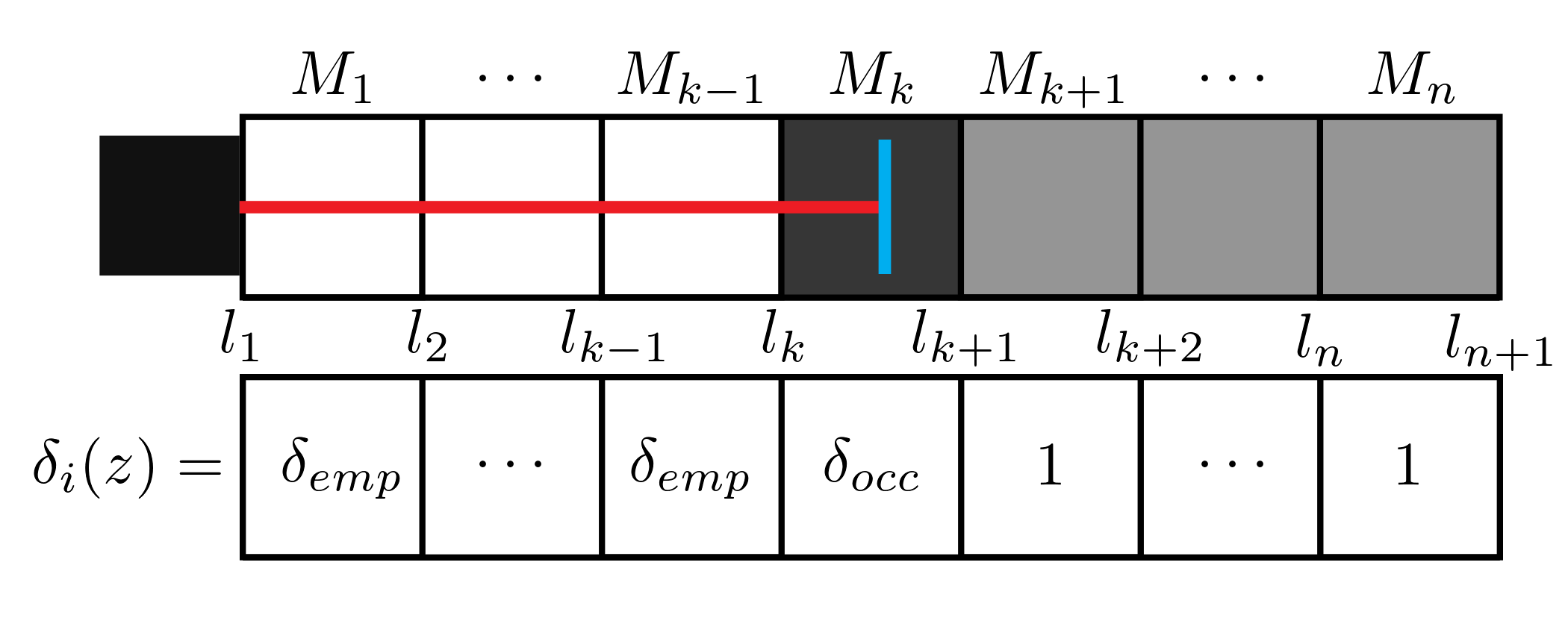}
\caption{Illustration of the key idea behind the proof of Lemma~\ref{lemma:F_z_constant}. The sensor beam (red) hits an obstacle (blue) in cell $M_k$. The value of the odds ratio inverse sensor model is shown at the bottom of the figure.}
\label{fig:fz_constant}
\end{figure}

Let $F(z) = \sum_{i = 1}^{n} f(\delta_i(z), r_i)$. The following lemma states that $F(z)$ is a piecewise-constant function.
\begin{lemma}[Piecewise Constant Summation]
\label{lemma:F_z_constant}
 The function $F(z)$ is piecewise constant. In particular, if $z$ lies in the $k$-th cell,
i.e.  $l_k\le z < l_{k + 1}$, then
$F(z) = C_k$ where $C_k = f(\delta_{occ}, r_k) + \sum_{i < k} f(\delta_{emp}, r_i)$.
\end{lemma}

\begin{proof}
For $i < k$, a measurement of $z$ implies that the beam has passed through $M_i$. Therefore $M_i$ should be empty and $\delta_i(z) = \delta_{emp}$.
By definition, 
$\delta_k(z) = \delta_{occ}$. 
A measurement of $z$ also indicates that the beam stops at cell $k$ which gives no information about cells $M_j$ for $j > k$, so $f(\delta_j(z), r_j) = f(1, r_j) = 0$.
Therefore, each term of $F(z)$ is constant for $l_k \leq z \leq l_k +1$ and the sum is equal to the desired $C_k$, proving the lemma.\qed
\end{proof}

Lemma~\ref{lemma:F_z_constant} shows that the function $F(z)$ changes its value only at the cell boundaries, as $z$ increases. (See Figure~\ref{fig:fz_constant} for an illustration.)
If we compute the mutual information between a range measurement and all the cells it can intersect at once, we can take advantage of this property to turn the integration in Equation~\eqref{eqn:mi_original} into a sum.

Using Lemma~\ref{lemma:F_z_constant}, we prove Theorem~\ref{theorem:exact_shannon_mi} as follows:

\begin{proof}{(Theorem~\ref{theorem:exact_shannon_mi})}
We begin with the total Shannon mutual information for one range measurement as stated in Equation~\eqref{eqn:mi_sum}.
We substitute the definition of Shannon mutual information provided in Equation~\eqref{eqn:mi_original}, and rearrange to reveal the sum described in Lemma~\ref{lemma:F_z_constant}.
We arrive at the following expression:
\begin{equation}
\label{eqn:mi_deriv_1}
\begin{split}
&I(M'; Z) = \sum_{i=1}^{n} I(M_i; Z) = \sum_{i=1}^{n} \int_z P(z)f(\delta_i(z), r)dz\\
&=\sum_{i=1}^{n} \int_z \sum_{j=0}^{n} P(e_j)P(z|e_j)f(\delta_i(z), r)dz\\
&=\sum_{j=0}^{n} P(e_j)\int_z P(z|e_j)\left(\sum_{i=1}^{n}f(\delta_i(z), r_i)\right)dz\\
&=\sum_{j=0}^{n} P(e_j)\int_z P(z|e_j)F(z) dz.
\end{split}
\end{equation}
Inspired by the result of Lemma~\ref{lemma:F_z_constant}, we divide the integration over $z$ into a sum of multiple integration intervals across each cell boundary.
This allows us to isolate the described term that is constant across each cell as follows:
\begin{equation}
\label{eqn:mi_deriv_2}
\begin{split}
I(M'; Z) &= \sum_{j=0}^{n} P(e_j)\sum_{k = 1}^{n}\int_{l_k}^{l_{k+1}} P(z|e_j)C_k dz\\
&=\sum_{j=0}^{n}\sum_{k = 1}^{n} P(e_j)C_k \int_{l_{k}}^{l_{k + 1}} P(z|e_j) dz\\
&=\sum_{j=0}^{n}\sum_{k = 1}^{n} P(e_j)C_k G_{k, j}.\\
\end{split}
\end{equation}
This completes the proof.\qed
\end{proof}

\subsection{Computational Complexity of the FSMI Algorithm}
\label{sec:fsmi_complexity}
We study the time complexity of Algorithm \ref{alg:i(m, z)} with respect to, $n$, the number of cells that a single range measurement intersects. The result is stated in the following theorem:
\begin{theorem}
\label{theorem:fsmi_complexity}
The time complexity of Algorithm~\ref{alg:i(m, z)} is $O(n^2)$.
\end{theorem}

The proof of Theorem~\ref{theorem:fsmi_complexity} is straightforward with the following intermediate results:

\begin{lemma}
$P(e_j), 0\leq j \leq n$ can be computed altogether in $O(n)$ with Algorithm~\ref{alg:pei}.
\end{lemma}

\begin{lemma}
$C_k, 1\le k \le n$ can be computed altogether with Algorithm~\ref{alg:c_k}.
\end{lemma}

\begin{lemma}
The standard normal CDF, $\Phi(\cdot)$, can be evaluated with a look-up table. $G_{k, j}$ can be evaluated in $O(1)$ as follows:
\begin{equation}
G_{k, j} = \Phi\left(\frac{l_{k + 1} - \mu_j}{\sigma} \right) - \Phi\left(\frac{l_k - \mu_j}{\sigma} \right)
\end{equation} 
\end{lemma}

Note that, unlike the algorithm proposed in~\citep{julian2014mutual}, the FSMI algorithm does not perform any numerical integration.
As a result, the complexity of FSMI outperforms the algorithm in~\citep{julian2014mutual} by a factor of $\lambda_z$, the integration resolution.

\begin{remark}
FSMI has the same time complexity as the exact version of CSQMI.
\end{remark}

Furthermore, since we assume the noise distribution has constant $\sigma$, we can directly precompute $\Phi(\frac{\cdot}{\sigma})$ to avoid one division operation per query.

Note that in addition to tabulating $\Phi(\cdot)$, we also build look-up tables for $f(\delta_{occ}, r_i)$ and $f(\delta_{emp}, r_i)$ for all $i$ rather than evaluate them based on their definition as that can be computationally expensive. Specifically, we precompute $f(\delta_{occ}, r_i)$ and $f(\delta_{emp}, r_i)$ for a discrete set of values of $r_i$ and store the results in a look-up table. We set $\delta_{emp}\delta_{occ} = 1$, following~\cite{julian2014mutual}; Since $f(\delta_{emp}, r_i) = f(\delta_{occ}, 1 / r_i)$, a single look-up table suffices.

%

\subsection{Efficient Implementations via Gaussian Truncation}
\label{sec:fsmi_approx}

In~\citep{charrow2015csqmi}, the authors propose to approximate the CSQMI metric by setting the tail of the Gaussian noise distribution to zero. 
We apply their technique to FSMI. Specifically, let $\Delta$ be the truncation width. We approximate Equation~\eqref{eqn:fast_exact_mi} as follows:
\begin{equation}
\label{eqn:fsmi_approx}
I(M'; Z) = \sum_{j = 0}^{n}\sum_{k = j - \Delta}^{j + \Delta} P(e_j) C_k G_{k, j}.
\end{equation}

We refer to the variation of FSMI that applies the approximation described above as {\em Approx-FSMI}. We refer to the corresponding CSQMI algorithm with the same approximation as {\em Approx-CSQMI}. 

The complexity of evaluating Approx-FSMI is $O(n\Delta)$
This is the same complexity as Approx-CSQMI in~\citep{charrow2015csqmi}. 
Similar to CSQMI, the value of $\Delta$ can be as small as $3$ in practical problem instances~\citep{charrow2015csqmi}.

Even though both algorithms have the same asymptotic complexity, we argue that the Approx-FSMI algorithm can be implemented so that it requires fewer multiplications when compared to the Approx-CSQMI algorithm. 
Intuitively, this is because our computation in Equation~\eqref{eqn:fsmi_approx} has only one double summation while the approximate version of CSQMI in Equation~\eqref{eqn:csqmi_exact} has two similarly structured double summations of the same size. 
In this comparison, we omit other operations, {\em e.g.}, additions, because they are significantly cheaper than multiplications on both general purpose CPUs and FPGAs~\citep{rabaey2002digital, hennessy2011computer}. Other operations, such as $\log(\cdot)$, occur only a constant number of times in Equation~\eqref{eqn:fsmi_approx} and the approximate version of Equation~\eqref{eqn:csqmi_exact}.

\begin{theorem}[Number of Multiplications]
\label{theorem:num_mul}
Evaluating Approx-FSMI and Approx-CSQMI require $(\Delta + 3)n$ and $(2\Delta + 9)n$ multiplications, respectively.
\end{theorem}

Corrected
between measurement
UNTITLEDSAVEDSAVED
Type your title

\subsection{Shannon Mutual Information Under Uniform Measurement Noise Model}

\label{sec:fsmi_uniform}

Recall that the asymptotic time complexity of Approx-FSMI is $O(n\Delta)$. Here, we show that when the sensor noise is modeled by a uniform distribution rather than a Gaussian distribution, under reasonable technical assumptions, the Shannon mutual information can be evaluated in $O(n)$, independently of $\Delta$.

\begin{theorem}[Uniform-FSMI]

\label{theorem:uniform_fsmi}

Suppose that cells have constant width. Again, let the boundary of the $i$-th cell being $l_i$ and $l_{i+1}$ and for all $i$, $l_{i + 1} - l_i = \Delta L$, which is a constant. Suppose that the sensor noise model is uniform and that the limits are quantized onto cell boundaries:

\begin{equation}
P(Z | e_i) \sim U[l_i - H\Delta L,\; l_{i + 1} + H\Delta L],
\end{equation}
for $H\in\mathbb{Z}^+$ is a constant for the beam.

Let $D_i = \sum_{j\le i} C_j$ for $1\le i\le n$ and $D_i = 0$ otherwise. Then the Shannon mutual information between the beam and all the cells it intersects is

\begin{equation}
I(M'; Z) 
= \sum_{j=0}^n P(e_j) \frac{D_{j + H} - D_{j - H - 1}}{2H + 1}
\end{equation}

\end{theorem}

\begin{proof}

This proof follows the proof of 

Theorem \ref{theorem:exact_shannon_mi} until Eq. ~\eqref{eqn:mi_deriv_1}. 

There, we plug in the PDF of the uniform distribution:


\begin{equation}
\label{eqn:uniform}
\begin{split}
I(M'; Z) 
&=\sum_{j=0}^{n}\sum_{k = 1}^{n} P(e_j)C_k \int_{l_{k}}^{l_{k + 1}} P(z|e_j) dz\\
&=\sum_{j=0}^{n}P(e_j)\sum_{k = j - H}^{j + H}\frac{C_k}{2H+1}\\
&= \sum_{j=0}^n P(e_j) \frac{D_{j + H} - D_{j - H - 1}}{2H + 1}
\end{split}
\end{equation}







This completes the proof.\qed

\end{proof}

\begin{remark}

If the sensor does not strictly follow a uniform distribution, we can approximate it with a uniform distribution by matching the mean and variance of the two distributions. For example, if $P(z | e_i)\sim \mathcal{N}(\mu_i, \sigma)$


we can set
$H = \text{round}\left(\sqrt{3}\sigma  - 1 / 2\right)$.
\end{remark}

The algorithm to compute Uniform-FSMI is summarized in Algorithm~\ref{alg:mi_uniform}. Its complexity is stated in the following theorem:

\begin{theorem}[Time complexity of the Uniform-FSMI algorithm]
The time complexity of Algorithm~\ref{alg:mi_uniform} is $O(n)$.
\end{theorem}
\begin{proof}
Line~\ref{line:mi_uniform:pej} computes all of $P(e_j)$ in $O(n)$. Line~\ref{line:mi_uniform:C_k} computes all of $C_k$ in $O(n)$ as well. The for-loop from Line~\ref{line:mi_uniform:D_k:start} to Line~\ref{line:mi_uniform:D_k:end} finishes in $O(n)$ time. The last for-loop from Line~\ref{line:mi_uniform:I_start} to Line~\ref{line:mi_uniform:I_end} computes the Shannon mutual information in $O(n)$. Therefore, the total time complexity of Algorithm~\ref{alg:mi_uniform} is $O(n)$. \qed
\end{proof}

The time complexity of the Uniform-FSMI algorithm outperforms Approx-FSMI and Approx-CSQMI by a factor of $\Delta$. In the experiments presented in Section~\ref{sec:experiments}, we demonstrate how this translates to an additional speedup of the Shannon mutual information computation in practice.

\begin{algorithm}

\caption{The Uniform-FSMI algorithm}\label{alg:mi_uniform}

\begin{algorithmic}[1]

\Require $H$ and $r_i$ for $1\le i\le n$.

\State $I\leftarrow 0$

\State Compute $P(e_j)$ for $0\leq j \leq n$ with Algorithm \ref{alg:pei}. \label{line:mi_uniform:pej}

\State Compute $C_k$ for $1 \leq k \leq n$ with Algorithm \ref{alg:c_k}. \label{line:mi_uniform:C_k}

\\

$D_k \leftarrow 0$

\For{$k = 1$ \textbf{to} $n$}\label{line:mi_uniform:D_k:start}

    \State $D_k \leftarrow D_{k - 1} + C_k$ 

\EndFor \label{line:mi_uniform:D_k:end}

\For{$j=0$ \textbf{to} $n$}\label{line:mi_uniform:I_start}

    \State $I\leftarrow I + P(e_j) \frac{D_{\min(n, j + H)} - D_{\max(0, j - H - 1)}}
    {2H+1}$

\EndFor \label{line:mi_uniform:I_end}

\State \Return $I$

\end{algorithmic}

\end{algorithm}

ALL ALERTS


\section{Fast Shannon Mutual Information for 3D Environment using Run-Length Encoding}
\label{sec:fsmi_3d}

The algorithms and analysis provided in Section~\ref{sec:fast_mi} focused on two-dimensional environments. In this section, we are concerned with algorithms and representations that can handle three-dimensional environments. 
A natural extension of the two-dimensional occupancy map is a three-dimensional voxel map~\citep{roth1989building, moravec1996robot}. 
Indeed, the algorithms and analysis presented in Section~\ref{sec:fast_mi} readily extend to three-dimensional mapping problems using the voxel map. 
However, in most applications, the memory requirements for the voxel map exceed what is typically available on embedded computers today, which can be mounted on robots of smaller form factors. Hence, due to the size of the map, the algorithms based on voxel maps will be relatively inefficient when used as-is for three-dimensional mapping of large-scale environments.

Fortunately, in most applications, the three-dimensional space that a robot navigates has a special structure: the empty spaces are typically large and continuous. 
The OctoMap data structure in~\cite{octo} takes advantage of this structure. It compresses the map by using voxels of varying sizes. See Figure~\ref{fig:octomap_raytracing} for an illustration. Large homogeneous regions that would contain thousands of small cells can be represented with a single cell in an OctoMap representation. Thus, OctoMap representations have become widely used in mapping and exploration in three-dimensional environments~\citep{endres20143, whelan2015real, burri2015real}. 

To the best of our knowledge, there exists no prior work that studies the computation of mutual information between range measurements and the environment represented by an OctoMap. 

For efficient computation, we represent each measurement using the run-length encoding (RLE)~\citep{robinson1967results}. Consider a measurement beam, as shown in Figure~\ref{fig:octomap_raytracing}. Suppose we project this beam onto equally-sized ``virtual cells,'' as shown in Figure~\ref{fig:octomap_rle}. Then, the RLE encoding is a sequence of numbers that encodes each occupancy value together with the number of consecutive virtual cells with the same occupancy value. This simple compression method is valuable for representing measurements on the OctoMap structure in a way that we can apply the analysis presented in the previous section. 

\begin{figure}[!b]
\centering 
\subfigure[Ray-tracing on OctoMap]{\label{fig:octomap_raytracing}\includegraphics[width=0.7\columnwidth]{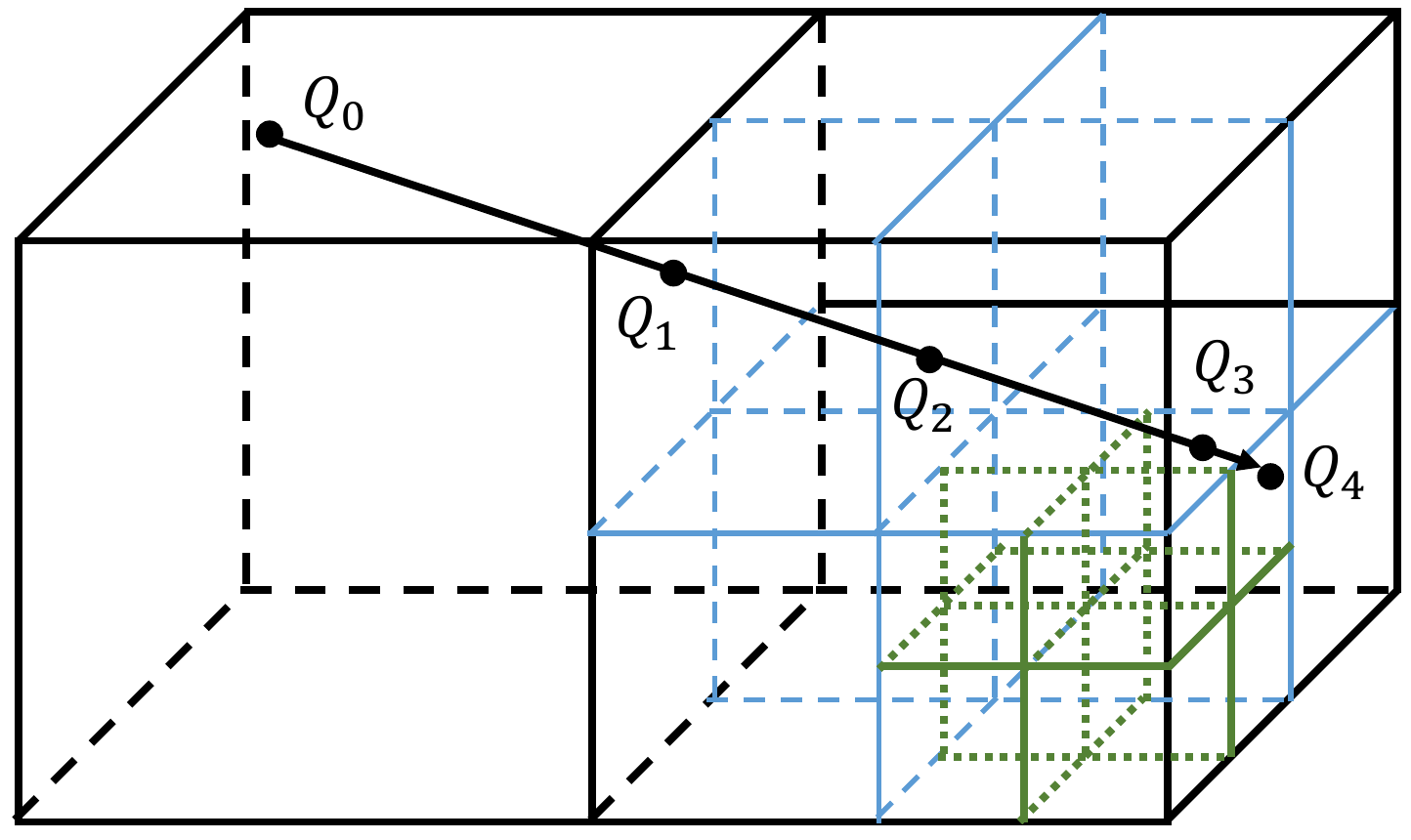}}
\subfigure[Ray-tracing result represented by run-length encoding (RLE)]{\label{fig:octomap_rle}\includegraphics[width=0.9\columnwidth]{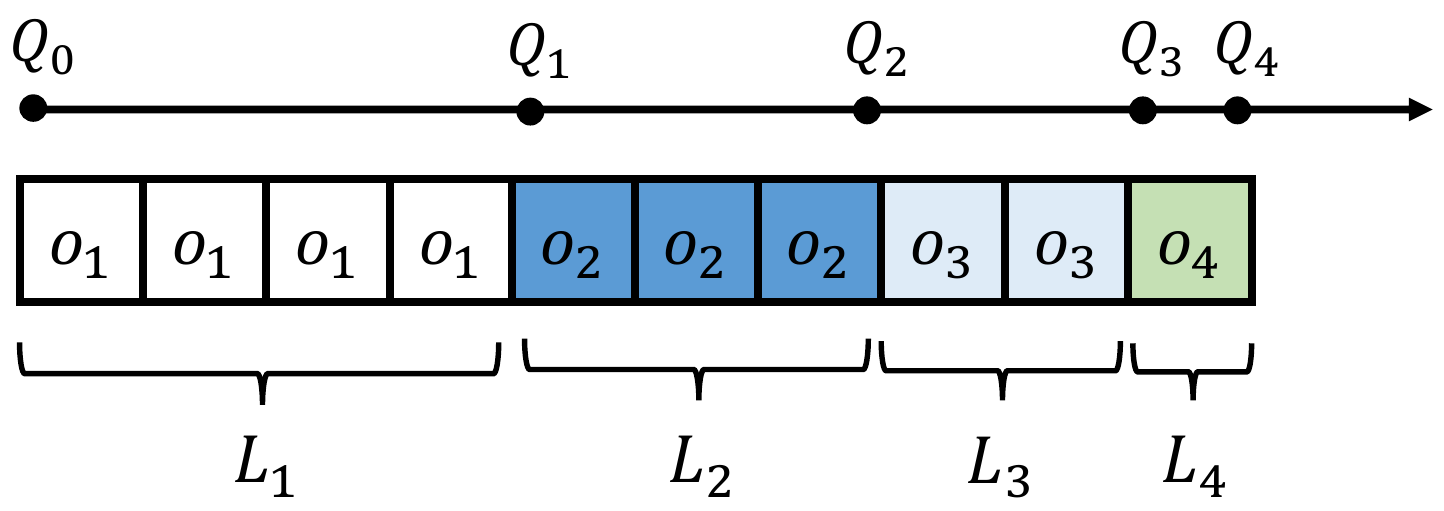}}
\caption{An illustrative example of how an OctoMap representation adaptively represent an environment, and how ray-tracing on the OctoMap representation results in an occupancy sequence represented by run-length encoding (RLE) format.}
\end{figure}

This section proposes a class of algorithms to compute the Shannon mutual information directly on this compressed sequence. We will show that the time complexity of this algorithm is linear with respect to the length of the compressed vector. In practice, this translates to significant savings in computation time.

This section is organized as follows. Section~\ref{sec:exact_fsmi_rle} formalizes the problem of adapting FSMI to run-length encoding and presents an efficient tabulation-based solution, namely the FSMI-RLE Algorithm. Section~\ref{sec:fsmi_rle_approx} discusses a numerical issue with the proposed solution and resolves it using Gaussian truncation, which we call the Approx-FSMI-RLE algorithm. Finally, Section~\ref{sec:fsmi_rle_uniform} introduces a Uniform-FSMI-RLE algorithm for mapping in 3D when the sensor noise follows a uniform distribution.

\subsection{FSMI-RLE Algorithm: Shannon Mutual Information on Occupancy Sequences Compressed by RLE}
\label{sec:exact_fsmi_rle}
In this section, we present the FSMI-RLE algorithm which computes the Shannon mutual information between a sensor beam and an array of cells whose occupancy values are compressed by Run-Length Encoding (RLE). After introducing our notation, Theorem~\ref{theorem:group_fsmi_rle} presents the main result. Algorithm~\ref{alg:exact_fsmi_rle} summarizes the computation according to the theorem.

Suppose the sequence consists of $n$ total number of virtual cells, divided into $n_r$ groups, each consisting of consecutive virtual cells with the same occupancy value. The $i$-th group contains cells with the same occupancy probability, $o_i\in(0, 1)$. We define the number of cells in the $i$-th group by $L_i\in \mathbb{Z}^{+}$. The total number of cells, $n$, relates to $L_i$ by $\sum_{i=1}^{n_r} L_i = n$. Thus, the vector of occupancy values for the range measurement is

\begin{equation}
\label{eqn:occup_rle}
\{\underbrace{o_1,\ldots, o_1}_{\textrm{repeated $L_1$ times}}, ..., \underbrace{o_{n_r},\ldots, o_{n_r}}_{\textrm{repeated $L_{n_r}$ times}}\}.
\end{equation}

Let $s_u = \sum_{i < u} L_i$ denote the index of the first cell in the $u$-th group of cells.
Let $P_E(u) = \prod_{i < u} (1-o_i)^{L_i}$ denote the probability that all the cells on the beam before the first cell of the $u$-th group of cells are empty. 

Also define $D_E(u)=\sum_{i< u} L_i f(\delta_{emp}, r_i)$. 

In addition, define the following:
\begin{eqnarray}
\alpha[x, L_u, L_v] = \sum_{j=0}^{L_u - 1}\sum_{k=0}^{L_v - 1} x^j \exp{\left(-\frac{(j-k)^2}{2\sigma'^2}\right)}, \\
\beta[x, L_u, L_v] = \sum_{j=0}^{L_u - 1}\sum_{k=0}^{L_v - 1} k\cdot x^j \exp{\left(-\frac{(j-k)^2}{2\sigma'^2}\right)},
\end{eqnarray}
where the variable $x \in [0, 1]$ is an occupancy probability, and the variables $L_u, L_v$ represent the block sizes of the $u$-th and $v$-th groups of cells. $\alpha, \beta$ are auxiliary terms used for calculating the mutual information for sensor beams that are occluded by a cell in the $u$-th group of cells but the measurements suggest that the beams fall in the $v$-th group of cells due to the sensor noise.

Figure~\ref{fig:fsmi_rle_notation} illustrates this notation.
Define $\bar{o}_u = 1-o_u$.
Let $w$ denote the width of a virtual cell, and define $\sigma' = \sigma / w$, where $\sigma$ is the standard deviation of the noise distribution.

\begin{theorem}[Shannon mutual information in RLE]
\label{theorem:group_fsmi_rle}
The Shannon mutual information $I(M';Z)$ of a measurement represented in RLE can be expressed as follows:
\begin{equation}
\begin{split}
\label{eqn:group_fsmi_rle_exact}
&I(M'; Z) = \sum_{u = 1}^{n_r}\sum_{v = 1}^{n_r} \frac{P_E(u) o_u}{\sqrt{2\pi}\sigma'}\Big(\\
&\quad (D_E(v) + f(\delta_{occ}, r_v))A[\bar{o}_u, L_u, L_v, s_u - s_v] +  \\ 
&\quad  f(\delta_{emp}, r_v) B[\bar{o}_u, L_u, L_v, s_u - s_v] \Big),
\end{split}
\end{equation}
where
\begin{equation*}\label{eqn:A_definition} A[x, L_u, L_v, t] = \sum_{j = 0}^{L_u - 1}\sum_{k = 0}^{L_v - 1} x^{j} \exp\left(-\frac{(j + t - k)^2}{2\sigma'^2}\right), \\
\end{equation*}

\begin{equation}
    \label{eqn:B_definition} B[x, L_u, L_v, t] = \sum_{j=0}^{L_u - 1}\sum_{k=0}^{L_v - 1} k\cdot x^j \exp{\left(-\frac{(j + t -k)^2}{2\sigma'^2}\right)}.
\end{equation}

In addition, $A[x, L_u, L_v, t]$ and $B[x, L_u, L_v, t]$ can be computed as
\begin{align}\label{eqn:compute_A}
\begin{split}
&\hspace{-0.1in} A[x, L_u, L_v, t] \\ &\hspace{-0.1in}=
\begin{cases}
x^{-t}\left(\alpha[x, L_u + t, L_v] - \alpha[x, t, L_v]\right), &\mbox{if } t \ge 1; \\
\alpha[x, L_u, L_v], &\mbox{if } t = 0; \\
\left(\alpha[x, L_u, L_v - t] - \alpha[x, L_u, -t]\right), &\mbox{if } t \le -1
\end{cases}
\end{split}
\end{align}
and
\begin{align}\label{eqn:compute_B}
\begin{split}
&\hspace{-0.1in} B[x, L_u, L_v, t]  \\ &\hspace{-0.1in}= 
\begin{cases}
x^{-t}\left(\beta[x, L_u + t, L_v] - \beta[x, t, L_v]\right) & \mbox{if } t \ge 1; \\
\beta[x, L_u, L_v] & \mbox{if } t = 0; \\
\beta[x, L_u, L_v - t] - \beta[x, L_u, -t] & \\ \qquad+ t\cdot A[x, L_u, L_v, t] & \mbox{if } t \le -1.
\end{cases}
\end{split}
\end{align}
\end{theorem}

\begin{figure}[!t]
\centering 
\includegraphics[width=0.9\columnwidth]{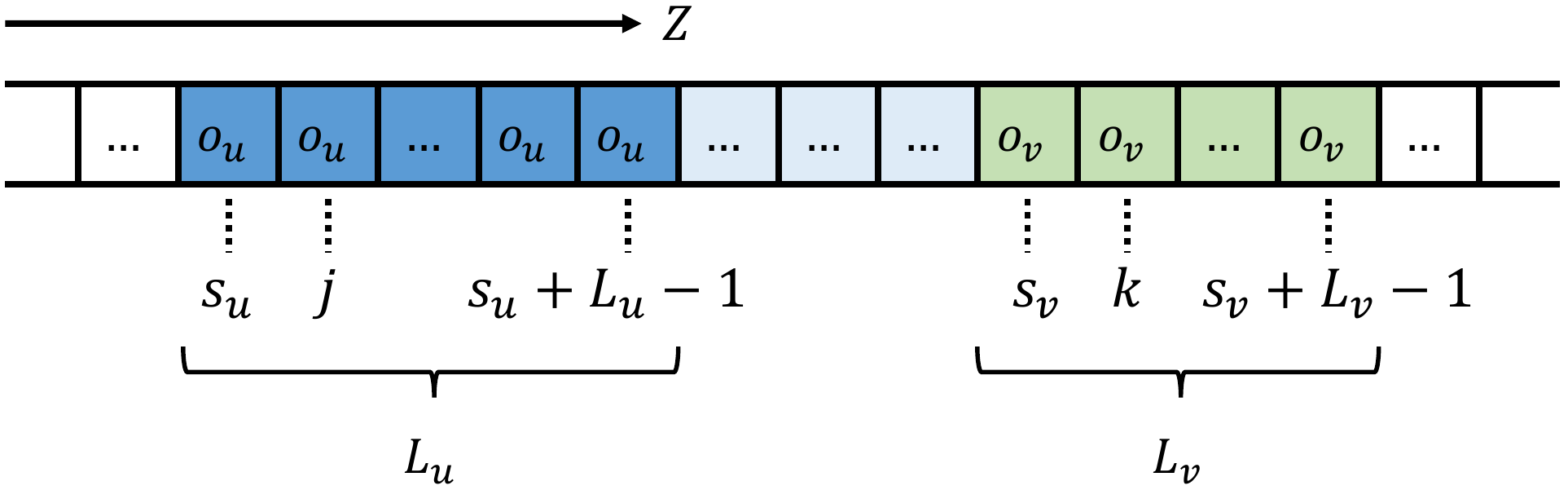}
\caption{Notation for FSMI computation on an occupancy sequence compressed by RLE.}
\label{fig:fsmi_rle_notation}
\end{figure}

\begin{proof}
First we prove the statement in Equation~\eqref{eqn:group_fsmi_rle_exact}. 
We reorganize Equation~\eqref{eqn:fast_exact_mi} into double sum over groups of cells defined in Equation~\eqref{eqn:occup_rle}:
\begin{equation}
\label{eqn:eqn_from_fsmi}
I(M'; Z) = \sum_{u = 1}^{n_r} \sum_{v = 1}^{n_r} \sum_{j = 0}^{L_u - 1}\sum_{k = 0}^{L_v - 1} P(e_{s_u + j}) C_{s_v + k} G_{s_u + j, s_v + k},
\end{equation}
where 
\begin{align}
\nonumber P(e_j) &=o_j\prod_{l<j}(1-o_l), \\
\nonumber C_k &= \sum_{l<k} f(\delta_{emp}, r_l) + f(\delta_{occ}, r_k), \\
\nonumber G_{k, j} & = \frac{1}{\sqrt{2\pi}\sigma'} \exp\left(-\frac{(k - j)^2}{2\sigma'^2}\right).
\end{align}
For a fixed $u$, for all $j$, we have
\begin{equation}
\label{eqn:p_e_u}
P(e_{s_u + j}) = P_E(u)(1-o_u)^{j}o_u.
\end{equation}
Similarly, 
\begin{equation}
\label{eqn:C_v}
C_{s_v + k} = D_E(v) + k\cdot f(\delta_{emp}, r_v) + f(\delta_{occ}, r_v).
\end{equation}
Substituting Equation~\eqref{eqn:p_e_u} and Equation~\eqref{eqn:C_v} into Equation~\eqref{eqn:eqn_from_fsmi} proves the statement in Equation~\eqref{eqn:group_fsmi_rle_exact}.

Second, we prove the statement in Equation~\eqref{eqn:compute_A}. 
When $t\ge 1$, 
\begin{equation}\nonumber
\begin{split}
    &A[x, L_u, L_v, t] = \sum_{j = 0}^{L_u - 1}\sum_{k = 0}^{L_v - 1} x^{j} \exp\left(-\frac{(j + t - k)^2}{2\sigma'^2}\right)\\
    =& \sum_{j = 0}^{L_u - 1}\sum_{k = 0}^{L_v - 1} x^{-t}x^{j + t} \exp\left(-\frac{(j + t - k)^2}{2\sigma'^2}\right)\\
    =&x^{-t}\sum_{j = 0}^{L_u - 1}\sum_{k = 0}^{L_v - 1} x^{j + t} \exp\left(-\frac{(j + t - k)^2}{2\sigma'^2}\right)\\
    =&x^{-t}\sum_{j = t}^{L_u + t - 1}\sum_{k = 0}^{L_v - 1} x^j \exp\left(-\frac{(j - k)^2}{2\sigma'^2}\right)\\
    =&x^{-t}\left(\alpha[x, L_u + t, L_v] - \alpha[x, t, L_v]\right).
\end{split}
\end{equation}

Similarly, when $t\le -1$,
\begin{equation}\nonumber
\begin{split}
    &A[x, L_u, L_v, t] = \sum_{j = 0}^{L_u - 1}\sum_{k = 0}^{L_v - 1} x^{j} \exp\left(-\frac{(j + t - k)^2}{2\sigma'^2}\right)\\
    =& \sum_{j = 0}^{L_u - 1}\sum_{k = -t}^{L_v - t - 1} x^{j} \exp\left(-\frac{(j - k)^2}{2\sigma'^2}\right)\\
    =&\alpha[x, L_u, L_v - t] - \alpha[x, L_u, -t].
\end{split}
\end{equation}

Third, we prove the statement in Equation~\eqref{eqn:compute_B}. The case for $t = 0$ is trivial. The case for $t\ge 1$ is similar to the case of $A$, above. When $t\le -1$, we have
\begin{equation}\nonumber
\begin{split}
&B[x, L_u, L_v, t] = \sum_{j=0}^{L_u - 1}\sum_{k=0}^{L_v - 1} k\cdot x^j \exp{\left(-\frac{(j + t -k)^2}{2\sigma'^2}\right)}\\
=&\sum_{j=0}^{L_u - 1}\sum_{k=0}^{L_v - 1} (k - t + t)\cdot x^j \exp{\left(-\frac{(j + t - k)^2}{2\sigma'^2}\right)}\\
=&\sum_{j=0}^{L_u - 1}\sum_{k=0}^{L_v - 1} (k - t)\cdot x^j \exp{\left(-\frac{(j-(k - t))^2}{2\sigma'^2}\right)} +\\
&\sum_{j=0}^{L_u - 1}\sum_{k=0}^{L_v - 1} t\cdot x^j \exp{\left(-\frac{(j+ t - k)^2}{2\sigma'^2}\right)}\\
= &\sum_{j=0}^{L_u - 1}\sum_{k=-t}^{L_v-t - 1} k\cdot x^j \exp{\left(-\frac{(j-k)^2}{2\sigma'^2}\right)} +\\
&\sum_{j=0}^{L_u - 1}\sum_{k=0}^{L_v - 1} t\cdot x^j \exp{\left(-\frac{(j+ t - k)^2}{2\sigma'^2}\right)}\\
= &\beta[x, L_u, L_v - t] - \beta[x, L_u, - t] + t\cdot A[x, L_u, L_v, t].
\end{split}
\end{equation}
\qed

\end{proof}

Theorem~\ref{theorem:group_fsmi_rle} motivates an efficient algorithm that balances time complexity and space complexity. Specifically, it is realized by the
{\em FSMI-RLE} algorithm, presented in Algorithms~\ref{alg:P_E_u}, \ref{alg:D_E_v}, and \ref{alg:exact_fsmi_rle}, which can rapidly evaluate Shannon mutual information with a time complexity that is independent of the number of virtual cells, $n$, and without requiring a large amount of memory.
In summary, the algorithm pre-computes and stores the functions $\alpha(x,L_u, L_v)$ and $\beta(x,L_u, L_v)$ in look-up tables, and uses Equation~\ref{eqn:group_fsmi_rle_exact} to combine these values to produce the Shannon mutual information. When tabulating the functions $\alpha(x,L_u, L_v)$ and $\beta(x,L_u, L_v)$ into look-up tables, the algorithm quantizes the occupancy probability variable $x \in [0,1]$ into a finite set of values\footnote{Let us note that, in many practical implementations of mapping algorithms, the occupancy values $o_u$ is quantized into a finite set of values from $[0, 1]$ for computational efficiency (see, \eg, \citep{GoogleCartographer}).}, which we denote by $\cal{X}$. 

It is easy to see that the functions $\alpha(x, L_u, L_v)$ and $\beta(x, L_u, L_v)$ are polynomials in $x$, hence continuous in $x$; thus, they can be approximated arbitrarily well by taking the resolution of ${\cal X}$ fine enough. Hence, the algorithm is still an exact algorithm for computing Shannon mutual information, as it can approximate the result with arbitrary accuracy. 

The correctness of Algorithm~\ref{alg:exact_fsmi_rle} is provided below. 
\begin{theorem}[Correctness of FSMI-RLE]
Algorithm~\ref{alg:exact_fsmi_rle} computes $I(M';Z)$ in Equation~\eqref{eqn:group_fsmi_rle_exact}.
\end{theorem}

\begin{proof}
Algorithm~\ref{alg:P_E_u} and Algorithm~\ref{alg:D_E_v} correctly computes $P_E(u)$ and $D_E(v)$; they are called in Lines~\ref{line:alg_exact_fsmi_rle_tab_alpha} and \ref{line:alg_exact_fsmi_rle_tab_beta} in Algorithm~\ref{alg:exact_fsmi_rle} to calculate $P_E(u)$ and $D_E(v)$ for all of $u, v, 1\le u, v\le n_r$. Then in Algorithm~\ref{alg:exact_fsmi_rle}, Line~\ref{line:alg_exact_fsmi_rle_for_u} and \ref{line:alg_exact_fsmi_rle_for_v} enumerate $u, v$ from $1$ to $n_r$; inside the loop, the computation follows Equation~\eqref{eqn:group_fsmi_rle_exact} except for the normalization, which happens outside the loop in Line \ref{line:alg_exact_fsmi_rle_I_norm}.
This completes the proof.\qed
\end{proof}
The time and space complexity of the algorithm are provided below. 
\begin{theorem}[Time complexity of FSMI-RLE]
The time complexity of the FSMI-RLE algorithm is $O(n_r^2)$. 
\end{theorem}
\begin{proof}
Algorithm~\ref{alg:P_E_u} and Algorithm~\ref{alg:D_E_v} only loops through $1$ to $n_r$ once; both have a complexity of $O(n_r)$. Although we put the tabulation of $\alpha$ and $\beta$ inside Algorithm~\ref{alg:exact_fsmi_rle} for completeness, in the actual implementation the table is precomputed once outside Algorithm~\ref{alg:exact_fsmi_rle}. The rest of the algorithm is a double for-loop to compute Equation~\eqref{eqn:group_fsmi_rle_exact}, resulting in a time complexity of $O(n_r^2)$. Therefore, the overall time complexity is $O(n_r^2)$.
\qed
\end{proof}
\begin{theorem}[Space complexity of FSMI-RLE]
The space complexity of the FSMI-RLE algorithm is $O(|\mathcal{X}| n^2)$.
\end{theorem}
\begin{proof}
Except for the two tables $\alpha$ and $\beta$, all the other variables in Algorithm~\ref{alg:exact_fsmi_rle} store at most $n_r$ real numbers. Each of $\alpha$ and $\beta$ tabulates against $x, L_u, L_v$. Note that $x$ is quantized into $|\mathcal{X}|$ entries and $1\le L_u, L_v\le n_r$; hence, $\alpha$ and $\beta$ each stores exactly $|\mathcal{X}|n_r^2$ real values. Therefore, the space complexity of Algorithm~\ref{alg:exact_fsmi_rle} is $O(|\mathcal{X}|n_r^2)$. This completes the proof.\qed
\end{proof}

First, note that the time complexity depends only on the number of groups of cells $n_r$, \ie, the size of the run-length encoding, but not on the number of virtual cells $n$. Note that $n_r$ is significantly smaller than $n$ because of the run-length encoding compression. Hence, reducing the time complexity from $O(n^2)$ to $O(n_r^2)$ translates to substantial savings in mutual information computation when the OctoMap achieves reasonable compression of the three-dimensional environment.

Second, let us note that closed-form solutions to $A$ and $B$ in Equation~\eqref{eqn:A_definition} and \eqref{eqn:B_definition} may reduce space complexity of the algorithm. However, the authors were unable to find closed-form solutions without any approximations. With approximation, we are able to derive closed-form solutions. Unfortunately, despite their constant time and space complexity the closed form solutions are so complex that its runtime exceeds that of the tabulated algorithms in all practical scenarios that we considered. See Appendix~\ref{sec:appendix_analytic} for our closed-form solutions. Finding closed-form solutions that can be evaluated more efficiently remains an open problem. 

Third, unfortunately, the FSMI-RLE algorithm suffers from numerical issues in some problem instances. Consider a case when there are two groups each with a very large number of virtual cells, say $L_u$ and $L_v$ are very large numbers for some $u$ and $v$, $u > v$. Recall that $A[\bar{o}_u, L_u, L_v, t]$ is the product of $\left(\bar{o}_u\right)^{-t}$ and $\alpha[\bar{o}_u, L_u + t, L_v] - \alpha[\bar{o}_u, t, L_v]$. Since $t = s_u - s_v \ge L_v$, the variable $t$ also becomes large. Hence, when $\bar{o}_u$ is close to zero, $\bar{o}_u^{-t}$ becomes a very large number. At the same time, $\left(\alpha[\bar{o}_u, L_u + t, L_v] - \alpha[\bar{o}_u, t, L_v]\right)$ becomes a very small number. As a result, in some cases, the multiplication of these numbers cannot be completed correctly due to the precision limitations of multiplication on computing hardware. The authors believe that the numerical issues can be avoided by tabulating the functions $A$ and $B$ into look-up tables, instead of the functions $\alpha$ and $\beta$. However, unfortunately, tabulating $A$ and $B$ requires orders of magnitude more memory, due to the additional integer variable $t$ that they encompass. Instead, in the next section, we propose an approximate algorithm for computing Shannon mutual information, based on truncating of the Gaussian distribution, which avoids the numerical issues while maintaining computational efficiency. The key insight behind the numerical stability of the approximation is that the variable $t$ never becomes too large as a result of the truncation.

\begin{algorithm}[!tb]
\caption{Evaluate $P_E(u)$ for $1 \le u \leq n_r$}\label{alg:P_E_u}
\begin{algorithmic}[1]
\Require $o_u$, $L_u$ for $1 \leq u \leq n_r$
\State $P_E(1)\leftarrow 1$
\For{$u=2$ \textbf{to} $n_r$}
    \State $P_E(u)\leftarrow P_E(u - 1)(1-o_{u - 1})^{L_{u - 1}}$
\EndFor
\State \Return $P_E(u)$ for $1\le u\le n_r$
\end{algorithmic}
\end{algorithm}
\begin{algorithm}[!tb]
\caption{Evaluate $D_E(v)$ for $1 \le v \leq n_r$}\label{alg:D_E_v}
\begin{algorithmic}[1]
\Require $r_v$, $L_v$ for $1 \leq v \leq n_r$
\State $D_E(1)\leftarrow 0$
\For{$v=2$ \textbf{to} $n_r$}
    \State $D_E(v)\leftarrow D_E(v - 1) + L_{v - 1}f(\delta_{emp}, r_{v - 1})$
\EndFor
\State \Return $D_E(v)$ for $1\le v\le n_r$
\end{algorithmic}
\end{algorithm}
\begin{algorithm}[!tb]
\caption{The FSMI-RLE algorithm}\label{alg:exact_fsmi_rle}
\begin{algorithmic}[1]
\Require Noise standard derivation $\sigma$, cell width $w$ and $s_i, L_i, o_i, r_i$ for $1\le i\le n_r$.
\State $I\leftarrow 0$
\State $\sigma' = \sigma / w$
\State Compute $P_E(u)$ for $1\leq u \leq n_r$ with Algorithm~\ref{alg:P_E_u}.
\State Compute $D_E(v)$ for $1 \leq v \leq n_r$ with Algorithm~\ref{alg:D_E_v}.
\State \label{line:alg_exact_fsmi_rle_tab_alpha} Tabulate $\alpha[x, L_u, L_v]$. 
\State \label{line:alg_exact_fsmi_rle_tab_beta} Tabulate $\beta[x, L_u, L_v]$.
\For{\label{line:alg_exact_fsmi_rle_for_u}$u=1$ \textbf{to} $n_r$}
    \State $x\leftarrow \mathrm{floor}((1-o_u) / o_{res})o_{res}$
    \For{\label{line:alg_exact_fsmi_rle_for_v}$v=1$ \textbf{to} $n_r$}
        \State Compute $A[x, L_u, L_v, t]$ based on Equation~\eqref{eqn:compute_A}.
        \State Compute $B[x, L_u, L_v, t]$ based on Equation~\eqref{eqn:compute_B}.
        \State $I\leftarrow I + P_E(u) o_u((D_E(v)  + f(\delta_{occ}, r_v)) A  + f(\delta_{emp}, r_v)) B)$
    \EndFor
\EndFor
\State \label{line:alg_exact_fsmi_rle_I_norm}$I\leftarrow I / (\sqrt{2\pi}\sigma')$
\State \Return $I$
\end{algorithmic}
\end{algorithm}

\subsection{Approx-FSMI-RLE Algorithm: Approximating Shannon Mutual Information in RLE via Truncation}
\label{sec:fsmi_rle_approx}

In this section, we present an approximate version of the algorithm presented in the previous section. The approximation is achieved by truncating the Gaussian distribution, similar to the approximation presented in Section~\ref{sec:fsmi_approx}. In brief, we wish to obtain an efficient algorithm that evaluates Shannon mutual information as in Equation~\eqref{eqn:fsmi_approx}, which was obtained from Equation~\eqref{eqn:fast_exact_mi} by setting $G_{k, j}=0$ for all $k,j$ with $|k-j|>\Delta$. 

In the rest of this section, we first state the truncated parallel of Theorem~\ref{theorem:group_fsmi_rle}. Then, we present the Approx-FSMI-RLE algorithm, which approximates the FSMI-RLE algorithm via Gaussian truncation. Next, we prove the correctness of the Approx-FSMI-RLE algorithm and then analyze its computational complexity. 

Recall that, for every two groups of virtual cells, say $u$ and $v$, the variables $s_u$, $s_v$ denote the index for the first virtual cell of its group and the variables $L_u$, $L_v$ denote the number of virtual cells in that group. When we evaluate Shannon mutual information using Equation~\eqref{eqn:group_fsmi_rle_exact}, most terms in $A$ and $B$ will be equal to zero due to Gaussian truncation. There are three cases to consider. First, two groups $u$ and $v$ might be further than $\Delta$ away, in which case all terms in $A$ and $B$ evaluate to zero. See Figure~\ref{fig:block_truncation} for an illustration. Second, two groups are close to each other and long enough, so that only a subset of the terms are zero.  See Figure~\ref{fig:rle_truncation_effective} for an illustration. In this case, we consider their sub-groups with non-zero elements. We denote the starting index of the sub-groups by $s_u'$, $s_v'$, and we denote the length of these subgroups by $L_u'$, $L_v'$. Third, the two groups closer and shorter than $\Delta$. See Figure~\ref{fig:rle_truncation_noeffect} for an illustration. In this case, all terms are non-zero. 

The theorem below presents the main result of this section. 
This result establishes a concise formula to compute the Shannon mutual information after the truncation of the Gaussian distribution. Most importantly, it saves computation by eliminating terms that evaluate to zero. In particular, it provides the formulae to compute the variables $s_u'$, $s_v'$, $L_u'$ and $L_v'$, which we referenced above. 

\begin{theorem}[Approx-FSMI-RLE algorithm]
\label{theorem:approx_fsmi_rle}
If $G_{k, j}=0$ when $|k-j|>\Delta$, the Shannon mutual information can be evaluated as
\begin{equation}\label{eqn:approx_fsmi_rle:three_parts}
I(M'; Z) = \frac{I_{u < v} + I_{u > v} + I_{u=v}}{\sqrt{2\pi}\sigma'},
\end{equation}
where
\begin{equation}\label{eqn:approx_fsmi_rle:less_than}
\begin{split}
&I_{u<v} = \sum_{u = 1}^{n_r}\sum_{\substack{v > u\\ s_v < s_u + L_u + \Delta}} P_E(u) (1-o_u)^{s_u'(u, v) - s_u} o_u\\
&\Big(\,(D_E(v)f(\delta_{emp}, r_v) + \\
&f(\delta_{occ}, r_v)) A[x, L_u'(u, v), L_v'(u, v), s_u - s_v] + \\
& f(\delta_{emp}, r_v) B[x, L_u'(u, v), L_v'(u, v), s_u - s_v] \Big),
\end{split}
\end{equation}
\begin{equation} \label{eqn:approx_fsmi_rle:more_than}
\begin{split}
&I_{u>v} = \sum_{u = 1}^{n_r}\sum_{\substack{v < u\\ s_u < s_v + L_v + \Delta}} P_E(u)  o_u\\
&\Big(\,(D_E(v) + (s_v'(u, v) - s_v)f(\delta_{emp}, r_v) + \\
&f(\delta_{occ}, r_v)) A[x, L_u'(u, v), L_v'(u, v), s_u - s_v] + \\
& f(\delta_{emp}, r_v) B[x, L_u'(u, v), L_v'(u, v), s_u - s_v] \Big),
\end{split}
\end{equation}
and
\begin{equation}\label{eqn:approx_fsmi_rle:equal}
\begin{split}
&I_{u=v} = \sum_{u = 1}^{n_r} P_E(u)o_u(D_E(u) + \\  & \quad f(\delta_{occ}, r_u))\theta[x, L_u] +
 f(\delta_{emp}, r_u)) \gamma[x, L_u]\Big),
\end{split}
\end{equation}
where 
\begin{equation*}
\begin{split}
\theta[x, L] &= \alpha[x, L, L],\\
\gamma[x, L] &= \beta[x, L, L].
\end{split}    
\end{equation*}
The $L_u'(u, v), L_v'(u, v), s_u'(u, v), s_v'(u, v)$ that appeared in the above equations are defined as:
\begin{equation}
    \begin{split}
    \label{eqn:su'}
        s_u'(u, v)=\left\{\begin{array}{ll} \max(s_u, s_v -\Delta) & u < v \\ s_u & u > v\end{array}\right.,
    \end{split}
\end{equation}
\begin{equation}
    \label{eqn:Lu'}
    \begin{split}
    L_u'(u, v) = \left\{\begin{array}{ll}s_u + L_u - s_u'(u, v) & u < v \\ \min(L_u, s_v + L_v + \Delta - s_u) & u > v\end{array}\right.
    \end{split}
\end{equation}

\begin{equation}
    \begin{split}
    \label{eqn:sv'}
        s_v'(u, v)=\left\{\begin{array}{ll} s_v & u < v \\ \max(s_v, s_u - \Delta) & u > v\end{array}\right.,
    \end{split}
\end{equation}
\begin{equation}
\label{eqn:Lv'}
    L_v'(u, v) = \left\{\begin{array}{ll}\min(L_v, s_u + L_u + \Delta - s_v) & u < v \\ s_v + L_v - s_v'(u, v) & u > v\end{array}\right.
\end{equation}
\end{theorem}

\begin{figure}[!b]
\centering{
\includegraphics[width= \columnwidth]{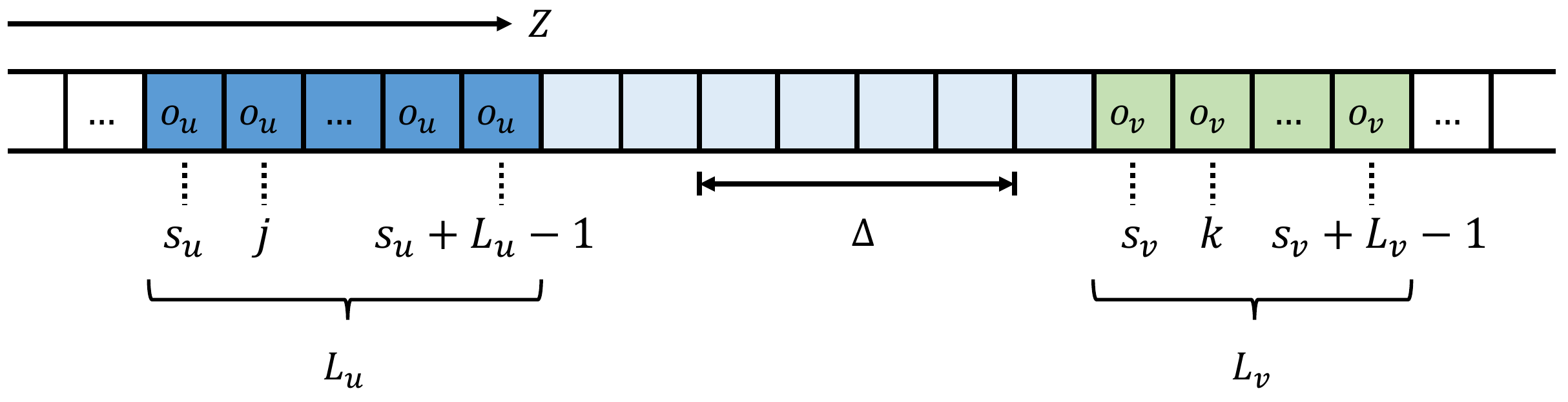}
}
\caption{When the condition in Lemma~\ref{lemma:non_zero_blocks_truncation} is violated, Gaussian truncation sets the Shannon mutual information contribution between these two blocks to zero; otherwise the Shannon mutual information contribution is non-zero.}
\label{fig:block_truncation}
\end{figure}

\begin{figure}[!tb]
\centering{
\subfigure[Start location and length of the blocks changed by Gaussian truncation]{\label{fig:rle_truncation_effective}\includegraphics[width= 0.9\columnwidth]{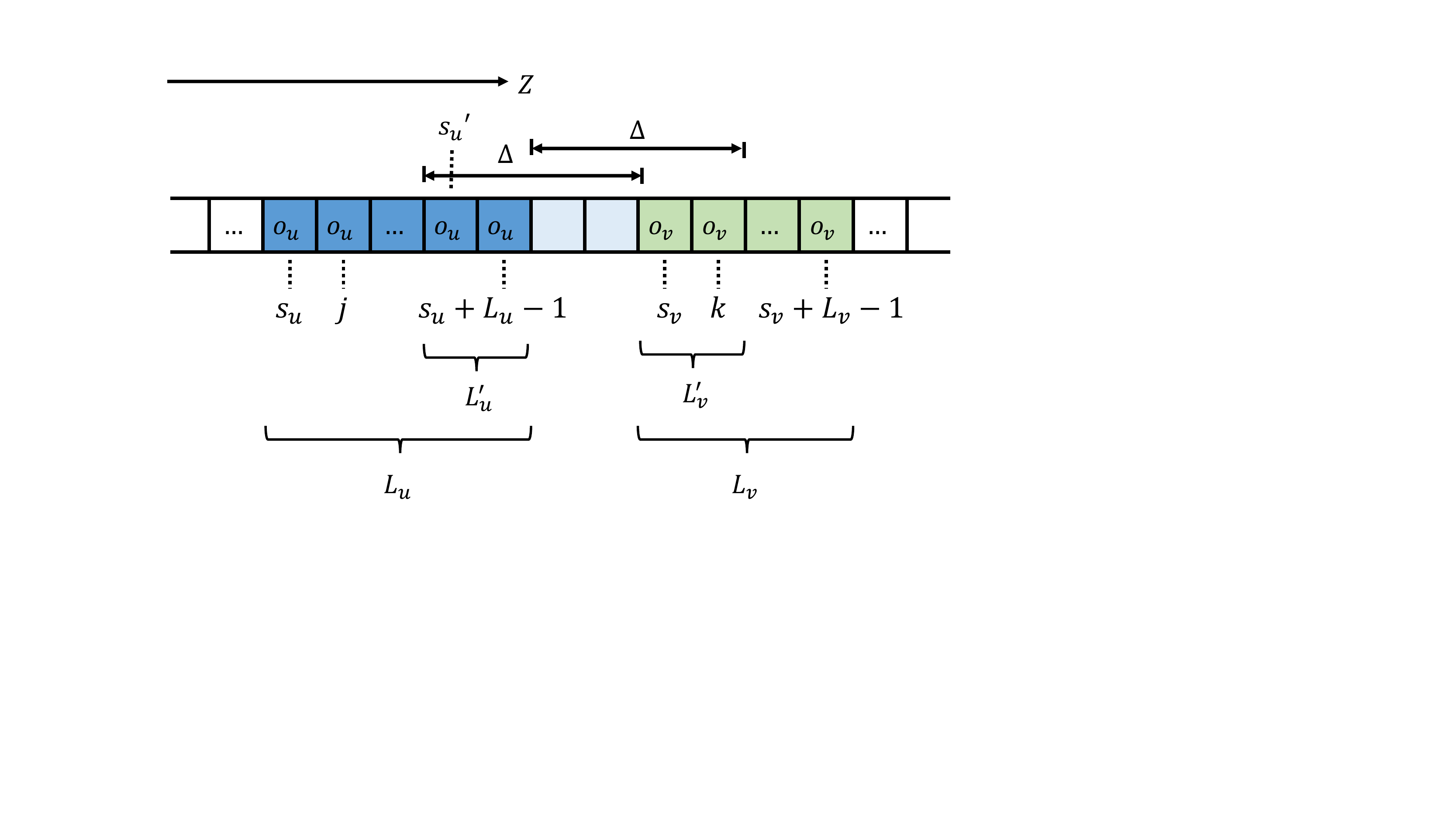}}
\subfigure[Start location and length of the blocks unaffected by Gaussian truncation]{\label{fig:rle_truncation_noeffect}\includegraphics[width= 0.9\columnwidth]{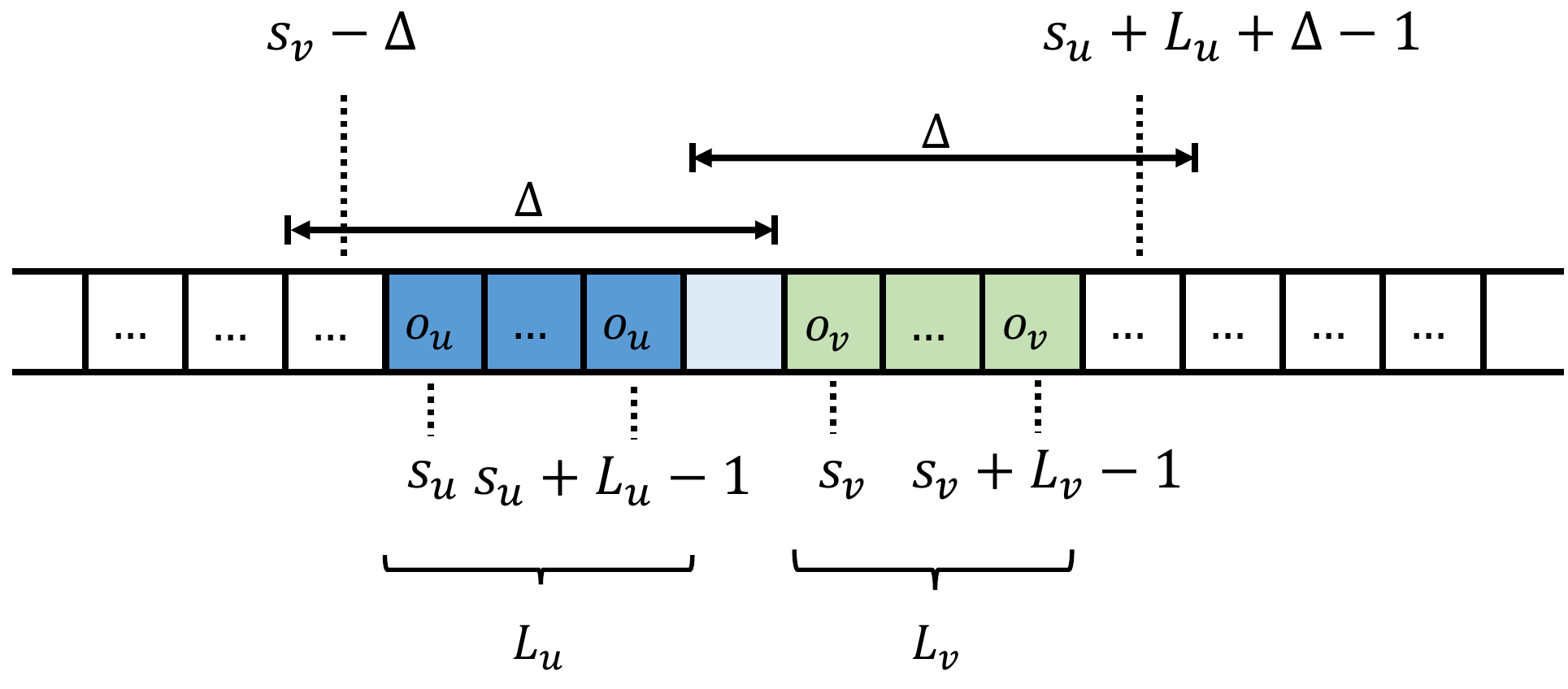}}
}
\caption{Gaussian truncation may change the start and the length of a pair of blocks of cell, depending on their relative distance and group length.}
\label{fig:truncation_position}
\end{figure}

To prove this theorem, we first establish a necessary and sufficient condition for $A[x, L_u, L_v, t]\ne 0$ and $B[x, L_u, L_v, t]\ne 0$ to hold after the truncation.

\begin{lemma}[Nonzero Mutual Information Block Pair after Gaussian Truncation]
\label{lemma:non_zero_blocks_truncation}
Let $t = s_u - s_v\ne 0$. If $G_{k, j} = 0$ when $|k - j| > \Delta$, then $A[x, L_u, L_v, t] \ne 0$ if and only if $u = v$, or $u < v, s_v < s_u + L_u + \Delta$ or $v < u, s_u < s_v + L_v + \Delta$. The same holds for $B[x, L_u, L_v, t]$.\end{lemma}

\begin{proof}
Figure~\ref{fig:block_truncation} illustrates the case when $u < v$. The case of $u > v$ is symmetric. When $u=v$ Gaussian truncation cannot set all terms in $A$ and $B$ to zero. This completes the proof.\qed
\end{proof}
This lemma enables us to skip computations for any pair of $(u, v)$ that does not satisfy the above condition.
Based on this lemma, now we prove Theorem~\ref{theorem:approx_fsmi_rle}. 
\begin{proof}(Theorem~\ref{theorem:approx_fsmi_rle})
We first prove the statement in Equations \eqref{eqn:su'}-\eqref{eqn:Lv'}. Since $u$ and $v$ are symmetric, we only discuss the case of $u < v$. See Figure~\ref{fig:truncation_position} for an illustration.

The formula for $s_u'(u, v)$, the new start position of the left block, is different depending on whether the two blocks belong to the case of Figure~\ref{fig:rle_truncation_effective} or Figure~\ref{fig:rle_truncation_noeffect}. In the former case, $s_u'(u, v)=s_v - \Delta$ while in the latter case $s_u'(u, v)=s_u > s_v - \Delta$. Combining them into one formula, we have $s_u'(u, v) = \max(s_u, s_v - \Delta)$.

With $s_u'(u, v)$, to get $L_u'(u, v)$ we simply subtract the number of truncated cells, $s_u'(u, v) - s_u(u, v)$,
from $L_u$, \ie, $L_u'(u, v) = L_u - (s_u'(u, v) - s_u) = s_u + L_u - s_u'(u, v)$.

Gaussian truncation does not affect the start position of the group of cells further away from the scanning position. Since in this case $u < v$, the start position of the $v$-th group stays the same, \ie, $s_v'(u, v) = s_v$.

The derivation for $L_v'(u, v)$ is similar to the derivation for $s_u'(u, v)$. If it is the case of Figure~\ref{fig:rle_truncation_effective}, we have $L_v'(u, v) = s_u + L_u + \Delta - s_v$. Otherwise $L_v'(u, v)=L_v < s_u + L_u + \Delta - s_v$. Combining the two cases, we get $L_v'(u, v) = \min(L_v, s_u + L_u + \Delta - s_v)$.

The same proof applies to the case of $u > v$. Hence we have proven Equations \eqref{eqn:su'}-\eqref{eqn:Lv'}.

The proof for the top part of this theorem follows Theorem~\ref{theorem:group_fsmi_rle}. After the Gaussian truncation, $s_u, s_v, L_u, L_v$ becomes $s_u'(u, v), s_v'(u, v), L_u'(u, v), L_v'(u, v)$. As a result, $P_E(u)$ will change when $s_u'\ne s_u$; let us denote the updated value by $P_E'(u)$. Similarly, the value of $D_E(v)$ also changes when $s_v'\ne s_v$ and we denote it by $D_E'(v)$.

For $P_E'(u)$, we note that $s_u'(u, v) = s_u$ when $u > v$ according to Equation~\eqref{eqn:su'}; hence, we have
\begin{equation}
\label{eqn:group_fsmi_truncation_P_E}
    P_E'(u) = \left\{\begin{array}{ll}P_E(u)(1-o_u)^{s_u'(u, v) - s_u} & u < v \\ 
    P_E(u) & u \ge v\end{array}\right..
\end{equation}

Similarly, for $D_E'(v)$, we derive the following equation based on Equation~\eqref{eqn:sv'} and Equation~\eqref{eqn:Lv'}:

\begin{equation}
\label{eqn:group_fsmi_truncation_D_E}
D_E'(v) = \left\{\begin{array}{ll}D_E(v) & u \le v \\ 
    D_E(v) + (s_v'(u, v) - s_v)f(\delta_{emp}, r_v) & u > v\end{array}\right..
\end{equation}
Substituting Equation~\eqref{eqn:group_fsmi_truncation_P_E} and ~\eqref{eqn:group_fsmi_truncation_D_E} into Equation~\eqref{eqn:group_fsmi_rle_exact} and separate the cases for $u< v$, $u=v$ and $u>v$, we prove the theorem. \qed
\end{proof}

Theorem~\ref{theorem:approx_fsmi_rle} motivates the Approx-FSMI-RLE algorithm, which we present in Algorithm~\ref{alg:approx_fsmi_rle}. The algorithm creates look-up tables for the functions $\alpha$, $\beta$, $\theta$, and $\gamma$ using Algorithms~\ref{alg:tab_alpha_theta} and \ref{alg:tab_beta_gamma} in Lines~\ref{line:approx_fsmi_rle:tabulate_alpha} and \ref{line:approx_fsmi_rle:tabulate_beta}. Then, the algorithm proceeds with computing Shannon mutual information based on Equation~\eqref{eqn:approx_fsmi_rle:three_parts} in Lines~\ref{line:approx_fsmi_rle:main_loop:begin}-\ref{line:approx_fsmi_rle:main_loop:end}. 
Line~\ref{line:approx_fsmi_rle:equal} handles the case when $u=v$ by applying the formula in Equation~\eqref{eqn:approx_fsmi_rle:equal}. Lines~\ref{line:approx_fsmi_rle:less_than:begin}-\ref{line:approx_fsmi_rle:less_than:end} handle the case when $u<v$ by applying the formula in Equation~\eqref{eqn:approx_fsmi_rle:less_than}. Finally, Lines~\ref{line:approx_fsmi_rle:more_than:begin}-\ref{line:approx_fsmi_rle:more_than:end} handle the case when $u>v$ by applying the formula in Equation~\eqref{eqn:approx_fsmi_rle:more_than}. 

The following theorem establishes the correctness of the Approx-FSMI-RLE algorithm. The proof of the algorithm is evident from the description above. 
\begin{theorem}[Correctness of Approx-FSMI-RLE]
Algorithm~\ref{alg:approx_fsmi_rle} computes I(M';Z) in Equations~\eqref{eqn:approx_fsmi_rle:three_parts}-\eqref{eqn:approx_fsmi_rle:more_than}. 
\end{theorem}

The following two theorems establish the computational complexity of the Approx-FSMI-RLE algorithm. 

\begin{theorem}[Time complexity of the Approx-FSMI-RLE algorithm]
The time complexity of the Approx-FSMI-RLE algorithm is $O(\Delta\,n_r)$.
\end{theorem}
\begin{proof}
First of all, the computation of $P_E(u)$ in Line~\ref{line:approx_fsmi_rle:Pe(u)} and $D_E(v)$ in Line~\ref{line:approx_fsmi_rle:De(v)} of Algorithm~\ref{alg:approx_fsmi_rle} takes $O(n_r)$ in total to complete. Then we note that the tabulation of $\alpha, \beta, \theta, \gamma$ in Line~\ref{line:approx_fsmi_rle:tabulate_alpha} and Line~\ref{line:approx_fsmi_rle:tabulate_beta} in practice takes place outside the algorithm for a single beam; these tables are filled only once at the start of an exploration task so that they will not be computed again for the evaluation of Shannon mutual information on any single beam once the exploration starts.

The for-loop over $u$ from Line~\ref{line:approx_fsmi_rle:main_loop:begin} to Line~\ref{line:approx_fsmi_rle:more_than:end} enumerates $u$ from $1$ to $n_r$. The computation consists of three parts, corresponding to $I_{u=v}$, $I_{u<v}$ and $I_{u>v}$ in Theorem~\ref{theorem:approx_fsmi_rle}. Line~\ref{line:approx_fsmi_rle:equal} handles the case of $v=u$ and it costs $O(1)$. Lines~\ref{line:approx_fsmi_rle:less_than:begin}-\ref{line:approx_fsmi_rle:less_than:end} handle the case of $u<v$ and loop through all $v\in\{v \mid v > u and s_v < s_u + L_u + \Delta\}$. Note that $s_u + L_u$ is the index of the start of the next group of cells right of the $u$-th group of cells. So we have $s_v\ge s_u + L_u$. Since $s_v < s_u + l_u + \Delta$, the number of valid $v$ is bounded by $(s_u + L_u + \Delta) - (s_u + L) = \Delta$. The computation inside the for-loop in Lines~\ref{line:approx_fsmi_rle:less_than:begin}-\ref{line:approx_fsmi_rle:less_than:end} is $O(1)$, hence the complexity of this for-loop is $O(\Delta)$. Similarly the complexity for the for-loop in Lines~\ref{line:approx_fsmi_rle:more_than:begin}-\ref{line:approx_fsmi_rle:more_than:end} that handle the case of $u>v$ is also $O(\Delta)$. So the overall complexity for the for-loop in Lines~\ref{line:approx_fsmi_rle:main_loop:begin}-\ref{line:approx_fsmi_rle:main_loop:end} is $O(\Delta\, n_r)$.

Therefore, the overall complexity of the whole algorithm is $O(n_r + \Delta\,n_r) = O(\Delta\, n_r)$. \qed
\end{proof}

\begin{theorem}[Space complexity of Approx-FSMI-RLE algorithm]
The space complexity of the Approx-FSMI-RLE algorithm is $O\left(|\mathcal{X}|(n+\Delta^2)\right)$.
\end{theorem}
\begin{proof}
The main memory of Algorithm~\ref{alg:approx_fsmi_rle} is used to store the tables $\alpha[x, L_u, L_v], \beta[x, L_u, L_v]$ and $\theta[x, L_u], \gamma[x, l_u]$. Compared with it, the memory for all other variables has at most $O(n_r)$ space complexity, which is negligible. For both $\alpha$ and $\beta$, we have $1\le L_u, L_v\le \Delta$ because of Gaussian truncation. For $\theta, \gamma$, we have $1\le L_u\le n$. Since $x\in\mathcal{X}$, the overall space complexity of Algorithm~\ref{alg:approx_fsmi_rle} is $O\left(|\mathcal{X}|(n+\Delta^2)\right)$.
\end{proof}

\begin{algorithm}[!tb]
\caption{The Approx FSMI-RLE algorithm}\label{alg:approx_fsmi_rle}
\begin{algorithmic}[1]
\Require Noise standard derivation $\sigma$, cell width $w$ and $s_i, L_i, o_i, r_i$ for $1\le i\le n_r$, quantization resolution $o_{res}$, Gaussian truncation width $\Delta$, maximal length of a group of cell $L_{M}$.
\State $I\leftarrow 0$
\State $\sigma' = \sigma / w$
\State Compute $P_E(u)$ for $1\leq j \leq n_r$ with Algorithm~\ref{alg:P_E_u}. \label{line:approx_fsmi_rle:Pe(u)}
\State Compute $D_E(v)$ for $1 \leq k \leq n_r$ with Algorithm~\ref{alg:D_E_v}.\label{line:approx_fsmi_rle:De(v)}
\State Tabulate $\alpha[x, L_u, L_v]$ and $\theta[x, L_u]$ with Algorithm~\ref{alg:tab_alpha_theta} where $L_{bound}=\Delta$ and $L_{max}=L_{M}$.
\label{line:approx_fsmi_rle:tabulate_alpha}
\State  Tabulate $\beta[x, L_u, L_v]$ and $\gamma[x, L_u]$ with Algorithm~\ref{alg:tab_beta_gamma} where $L_{bound}=\Delta$ and $L_{max}=L_{M}$. \label{line:approx_fsmi_rle:tabulate_beta}
\For{$u=1$ \textbf{to} $n_r$} \label{line:approx_fsmi_rle:main_loop:begin}

    \State $I\leftarrow I + P_E(u)o_u{Z}(D_E(u) + f(\delta_{occ}, r_v))\theta[x, L_u] + f(\delta_{emp}, r_u)) \gamma[x, L_u]$ \label{line:approx_fsmi_rle:equal}

    \For{$v\in\{v\mid v > u \textrm{ and } s_v < s_u + L_u + \Delta\}$} \label{line:approx_fsmi_rle:less_than:begin}
            \State $s_u' \leftarrow \max(s_u, s_v -\Delta)$
            \State $L_u' \leftarrow s_u + L_u - s_u(u, v)'$
            \State $s_v' \leftarrow s_v$
            \State $L_v' \leftarrow \min(L_v, s_u + L_u + \Delta - s_v)$

            \State $t\leftarrow s_u' - s_v'$
            \State $A = x^{-t}\left(\alpha[x, L_u' + t, L_v'] - \alpha[x, t, L_v']\right)$
            \State $B = x^{-t}\left(\beta[x, L_u' + t, L_v'] - \beta[x, t, L_v']\right)$
            \State $I\leftarrow I + P_E(u) (1-o_u)^{s_u' - s_u} o_u((D_E(v)  + f(\delta_{occ}, r_v)) A + f(\delta_{emp}, r_v) B)$
    \EndFor \label{line:approx_fsmi_rle:less_than:end}
    
    \For{$v\in\{v \mid v < u\textrm{ and } s_u < s_v + L_v + \Delta$\}} \label{line:approx_fsmi_rle:more_than:begin} 
            \State $s_u' \leftarrow s_u$
            \State $L_u' \leftarrow \min(L_u, s_v + L_v + \Delta - s_u)$
            \State $s_v' \leftarrow \max{(s_v, s_u - \Delta)}$
            \State $L_v' \leftarrow s_v + L_v - s_v'$
            \State $t\leftarrow s_u' - s_v'$
            \State $A = x^{-t}\left(\alpha[x, L_u' + t, L_v'] - \alpha[x, t, L_v']\right)$
            \State $B = x^{-t}\left(\beta[x, L_u' + t, L_v'] - \beta[x, t, L_v']\right)$
            \State $I\leftarrow I + P_E(u) o_u((D_E(v) + (s_v' - s_v)f(\delta_{emp}, r_v) + f(\delta_{occ}, r_v)) A  + f(\delta_{emp}, r_v)) B)$
    \EndFor \label{line:approx_fsmi_rle:more_than:end}
    
\EndFor
\State $I\leftarrow I / (\sqrt{2\pi}\sigma')$ \label{line:approx_fsmi_rle:main_loop:end}
\State \Return $I$
\end{algorithmic}
\end{algorithm}

\begin{algorithm}[!tb]
\caption{Tabulating $\alpha[x, L_u, L_v]$ and $\theta[x, L_u]$}
\label{alg:tab_alpha_theta}
\begin{algorithmic}[1]
\Require $x_{res}$, $\sigma'$, $L_{bound}$, $L_{max}$
\State $N_x\leftarrow \mathrm{floor}(1 / x_{res})$
\For{$i = 0$ \textbf{to} $N_x$}
    \State $x\leftarrow i\cdot x_{res}$
    \State $\alpha[x, 1, 1] \leftarrow 1$
    \For{$L_u = 2$ \textbf{to} $L_{bound}$}
        \State $\alpha[x, L_u, 1] = \alpha[x, L_u - 1, 1] + \exp{(-\frac{(L_u - 1)^2}{2\sigma'^2})}$
    \EndFor
    
    \For{$L_v = 2$ \textbf{to} $L_{bound}$}
        \State $\alpha[x, 1, L_v] = \alpha[x, 1, L_v - 1] + x^{L_v - 1}\exp{(-\frac{(1 - L_v)^2}{2\sigma'^2})}$
    \EndFor
    
    \For{$L_u = 2$ \textbf{to} $L_{bound}$}
        \For{$L_v = 2$ \textbf{to} $L_{bound}$}
            \State $\alpha[x, L_u, L_v] = \alpha[x, L_u, L_v - 1] + \alpha[x, L_u - 1, L_v] - \alpha[x, L_u - 1, L_v - 1] + x^{L_v - 1}\exp{(-\frac{(L_u - L_v)^2}{2\sigma'^2})}$
        \EndFor
    \EndFor
    
    \For{$L = 1$ \textbf{to} $L_{bound}$}
        \State $\theta[x, L] = \alpha[x, L, L]$
    \EndFor
    
    \For{$L = L_{bound} + 1$ \textbf{to} $L_{\max}$}
        \State $\theta[x, L] \leftarrow \theta[x, L - 1]$
        \For{$i = 1$ \textbf{to} $L - 1$}
            \State $\theta[x, L] \leftarrow \theta[x, L] + x^{L- 1}\exp{(-\frac{(i - L)^2}{2\sigma'^2})}$
            \State $\theta[x, L] \leftarrow \theta[x, L] + x^{i- 1}\exp{(-\frac{(i - L)^2}{2\sigma'^2})}$
        \EndFor
        \State $\theta[x, L] \leftarrow \theta[x, L] + x^{L- 1}$
    \EndFor
\EndFor

\State \Return $\alpha$ and $\theta$
\end{algorithmic}
\end{algorithm}

\begin{algorithm}[!tb]
\caption{Tabulating $\beta[x, L_u, L_v]$ and $\gamma[x, L_u]$}
\label{alg:tab_beta_gamma}
\begin{algorithmic}[1]
\Require $x_{res}$, $\sigma'$, $L_{bound}$, $L_{max}$
\State $N_x\leftarrow \mathrm{floor}(1 / x_{res})$
\For{$i = 0$ \textbf{to} $N_x$}
    \State $x\leftarrow i\cdot x_{res}$
    \State $\beta[x, 1, 1] \leftarrow 0$
    \For{$L_u = 2$ \textbf{to} $L_{bound}$}
        \State $\beta[x, L_u, 1] \leftarrow 0$
    \EndFor
    
    \For{$L_v = 2$ \textbf{to} $L_{bound}$}
        \State $\beta[x, 1, L_v] = \beta[x, 1, L_v - 1] + (L_v - 1)\exp{(-\frac{(1 - L_v)^2}{2\sigma'^2})}$
    \EndFor
    
    \For{$L_u = 2$ \textbf{to} $L_{bound}$}
        \For{$L_v = 2$ \textbf{to} $L_{bound}$}
            \State $\beta[x, L_u, L_v] = \beta[x, L_u, L_v - 1] + \beta[x, L_u - 1, L_v] - \beta[x, L_u - 1, L_v - 1] + (L_v - 1)x^{L_u - 1}\exp{(-\frac{(L_u - L_v)^2}{2\sigma'^2})}$
        \EndFor
    \EndFor
    
    \For{$L = 1$ \textbf{to} $L_{bound}$}
        \State $\gamma[x, L] = \beta[x, L, L]$
    \EndFor
    
    \For{$L = L_{bound} + 1$ \textbf{to} $L_{\max}$}
        \State $\gamma[x, L] \leftarrow \gamma[x, L - 1]$
        \For{$i = 1$ \textbf{to} $L - 1$}
            \State $\gamma[x, L] \leftarrow \gamma[x, L] + (i - 1)x^{L- 1}\exp{(-\frac{(i - L)^2}{2\sigma'^2})}$
            \State $\gamma[x, L] \leftarrow \gamma[x, L] + (L - 1)x^{i- 1}\exp{(-\frac{(i - L)^2}{2\sigma'^2})}$
        \EndFor
        \State $\gamma[x, L] \leftarrow \gamma[x, L] + (L - 1)x^{L- 1}$
    \EndFor
\EndFor

\State \Return $\beta$ and $\gamma$
\end{algorithmic}
\end{algorithm}



\subsection{Uniform-FSMI-RLE Algorithm: Shannon Mutual Information in RLE Assuming Uniform Distribution for Sensor Noise}
\label{sec:fsmi_rle_uniform}

In this section, we discuss an algorithm that computes Shannon mutual information on measurements in RLE assuming uniform distribution for the sensor noise. Unfortunately, the resulting algorithm is not substantially better than the Approx-FSMI-RLE algorithm. However, we still describe the key ideas behind the algorithm for the purposes of completeness. 

Let $\mathbbm{1}(\cdot)$ be the indicator function. We also employ the notations defined in Theorem~\ref{theorem:uniform_fsmi}. The following theorem presents the main result: 

\begin{theorem}[Uniform FSMI-RLE]
If the sensor measurement noise follows uniform distribution, i.e., 
\begin{equation}
\label{eqn:uniform_distribution_rle}
P(Z | e_i) \sim U[l_i - H w,\; l_{i + 1} + H w],
\end{equation}
for $H\in\mathbb{Z}^+$ and the width of the virtual cell $w$, then the Shannon mutual information between $Z$ and $M'$ can be evaluated as

\begin{equation}
\begin{split}\label{eqn:group_uniform_fsmi_rle}
&I(M'; Z) = \sum_{u = 1}^{n_r}\sum_{v = 1}^{n_r} \frac{P_E(u) o_u}{2H + 1}\Big(\\
&(D_E(v) + f(\delta_{occ}, r_v))\sum_{j=0}^{L_u - 1}\sum_{k=0}^{L_v - 1} k\cdot \bar{o}_{u}^{j} \mathbbm{1}(|j + t - k)|\le H)\\ 
&+ f(\delta_{emp}, r_v)) \sum_{j=0}^{L_u - 1}\sum_{k=0}^{L_v - 1} \bar{o}_{u}^{j}\mathbbm{1}(|j + t - k|\le H) \Big).
\end{split}
\end{equation}
\end{theorem}

This motivates us to focus on the computation of the following two terms:

\begin{equation}
\begin{split}
F[x, L_u, L_v, t] &= \sum_{j=0}^{L_u - 1}\sum_{k=0}^{L_v - 1} x^j \mathbbm{1}(|j + t - k)|\le H),\\
G[x, L_u, L_v, t] &= \sum_{j=0}^{L_u - 1}\sum_{k=0}^{L_v - 1} k\cdot x^j \mathbbm{1}(|j + t - k)|\le H).
\end{split}
\end{equation}

We apply the same tabulation techniques and algorithms presented in Section~\ref{sec:exact_fsmi_rle}, except we tabulate $F$ and $G$ and we use the Equation~\eqref{eqn:group_uniform_fsmi_rle} to compute Shannon mutual information. We call the resulting algorithm Uniform-FSMI-RLE.

The following theorem states the time complexity of the Uniform-FSMI-RLE algorithm:

\begin{theorem}[Time complexity of Uniform-FSMI-RLE]
\label{theorem:uniform_fsmi_rle_complexity}
If $F[x, L_u, L_v, t]$ and $G[x, L_u, L_v, t]$ can be evaluated in $O(1)$, the time complexity of the Uniform-FSMI-RLE algorithm is $O(n_r \, H)$. 
\end{theorem}

Note that the support of the uniform distribution itself only spans $2H\, w$ in length; therefore, its impact on the Shannon mutual information computation is similar to that of the Gaussian truncation, and $H$ is similar to $\Delta$ despite their definitions being different ($H$ is the radius of the uniform distribution, while $\Delta$ is the truncation of the Gaussian distribution). Therefore, the proof of Theorem~\ref{theorem:uniform_fsmi_rle_complexity} is trivial given the time complexity of the Approx-FSMI-RLE algorithm, so we skip it in this paper.

In practice, the magnitude of the value of $H$ is also comparable to $\Delta$; hence, the time complexity of the Uniform-FSMI-RLE algorithm, $O(n_r \, H)$ is not lower than the time complexity of the Approx-FSMI-RLE algorithm, which is $O(n_r \, \Delta)$.


\section{Experimental Results} \label{sec:experiments}
This section is devoted to our experiments. In Section~\ref{sec:experiments:fsmi-single}, we demonstrate the FSMI and Approx-FSMI algorithms in computational experiments with randomly generated occupancy values. Then, in Section~\ref{sec:experiments:fsmi-simulated}, we demonstrate the effectiveness of these algorithms for two-dimensional mapping in a synthetic environment. In Section~\ref{sec:plannar_mapping_exp}, we demonstrate the same algorithms in an experiment involving a 1/10-scale car-like robot. In all cases, we compare the proposed algorithms with the existing algorithms for computing the Shannon mutual information metric and existing algorithms for computing the CSQMI metric.

Then, we turn our attention to the proposed algorithms for run-length encoded occupancy sequences. In Section~\ref{sec:experiments:approx-fsmi-single}, we demonstrate the Approx-FSMI-RLE algorithm in computational experiments and show the speedup compared to the FSMI algorithm which does not support run length encoding. Following, in Section~\ref{sec:experiments:approx-fsmi-real-world}, we showcase the Approx-FSMI-RLE algorithm in a scenario involving a 1/10-scale car-like robot equipped with a Velodyne 16-channel laser range finder. 
All computational experiments use one core of an Intel Xeon E5-2695 CPU. All experiments involving the 1/10-scale car use a single core of the ARM Cortex-A57 CPU on the NVIDIA Tegra X2 platform. 

\subsection{Computational Experiments for 2D Mutual Information Algorithms} \label{sec:experiments:fsmi-single}
The first experiment studies the accuracy and the throughput of evaluating mutual information in computational experiments. We consider a scenario where the length of the beam is $\SI{10}{\meter}$ and the resolution of the occupancy grid is $\SI{0.1}{\meter}$. The occupancy values are generated at random. We set $\delta_{occ} = 1/\delta_{emp} = 1.5$.
We set the sensor's noise to be a normal distribution with a constant $\sigma = \SI{0.05}{\meter}$ regardless of the travel distance of the beam. 

First, we compare the run time and accuracy of following three algorithms: {\em (i)} the existing Shannon Mutual Information computation algorithm by~\cite{julian2014mutual} with integration resolution parameter set to $\lambda_z = \SI{0.01}{\meter}$, which we call the SMI algorithm, {\em (ii)} the FSMI algorithm, and {\em (iii)} the Approx-FSMI algorithm with truncation parameter set to $\Delta = 3$. To measure the accuracy of the algorithms, we compute the ground truth using the algorithm by~\cite{julian2014mutual} with integration resolution parameter set to  $\lambda_z = \SI{10}{\micro\meter}$. 
The results are summarized in Figure~\ref{fig:beam_mi_speed} and Figure~\ref{fig:beam_mi_error}. We observe that the FSMI algorithm computes Shannon mutual information more accurately than the SMI algorithm (with integration parameter set to $\lambda_z = \SI{0.01}{\meter}$) while running more than three orders of magnitude faster. This can be explained by the low numerical integration resolution when $\lambda_z=\SI{0.01}{\meter}$. The run time of Approx-FSMI is an additional $7$ times faster than FSMI at a small cost of accuracy. Still, the Approx-FSMI algorithm is more accurate than the SMI algorithm with $\lambda_z = \SI{0.01}{\meter}$. 

\begin{figure}[!t]
\centering 
\subfigure[Mean time ($\SI{}{\micro\second}$)]{\label{fig:beam_mi_speed}\includegraphics[height=0.27\columnwidth]{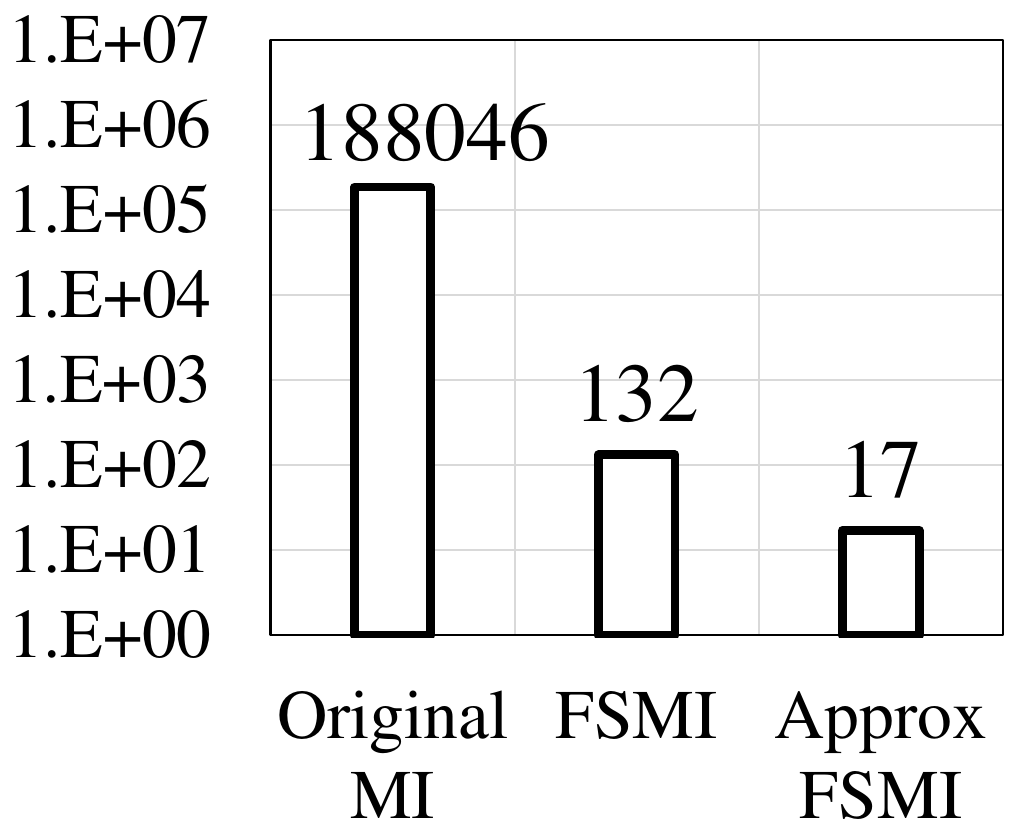}}
\subfigure[Mean relative error]{\label{fig:beam_mi_error}\includegraphics[height=0.27\columnwidth]{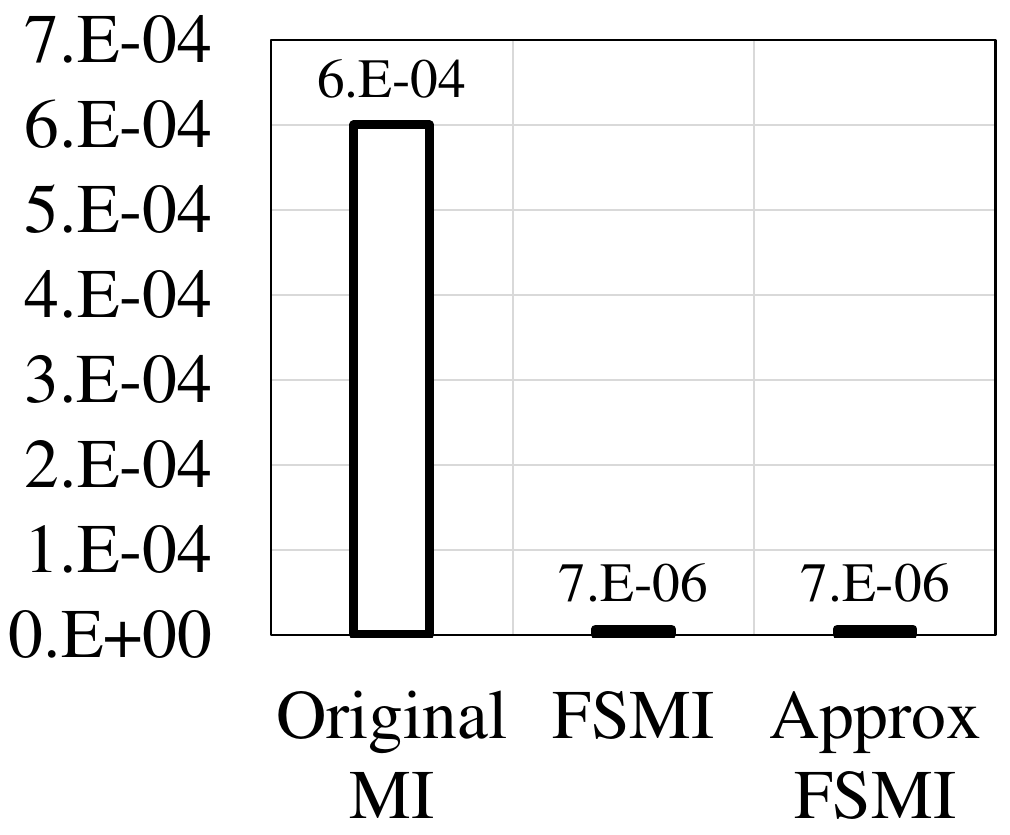}}
\subfigure[Mean time ($\SI{}{\micro\second}$)]{\label{fig:beam_csqmi_speed}\includegraphics[height=0.27\columnwidth]{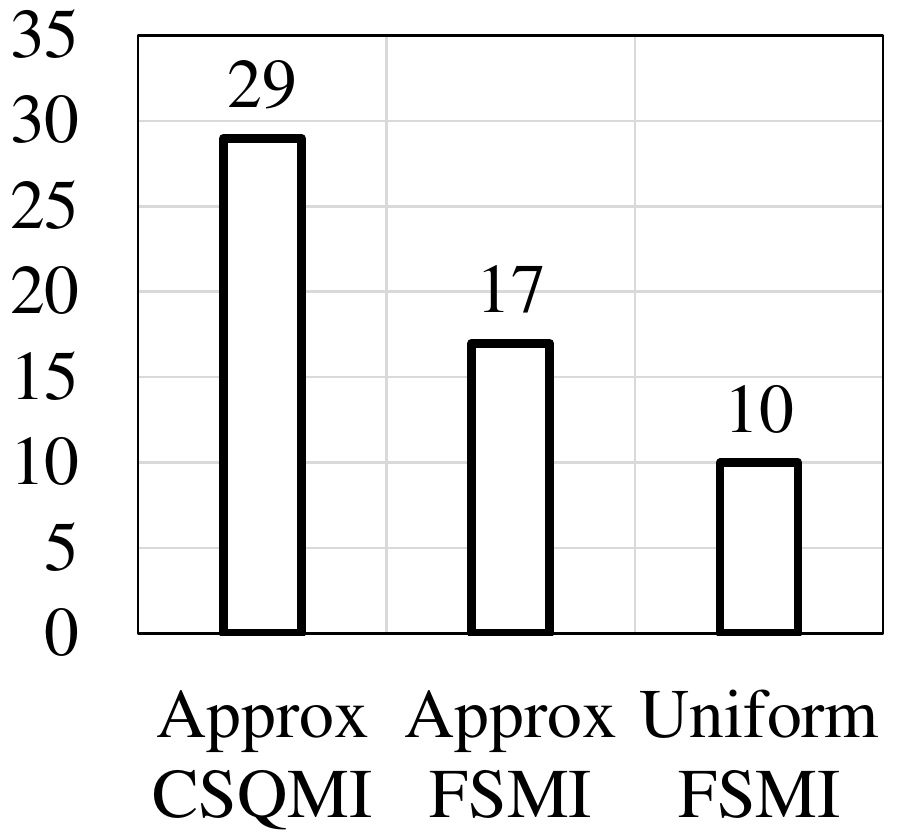}}
\caption{Speed \& relative error of different mutual information algorithms on a single beam.}
\label{fig:mi_vs_original_mi}
\end{figure}

Second, we compare the run time of the following three algorithms: {\em (i)} the Approx-CSQMI algorithm with truncation parameter set to $\Delta = 3$, {\em (ii)} the Approx-FSMI algorithm with truncation parameter set to $\Delta = 3$, {\em (iii)} the Uniform-FSMI algorithm. 
The results are shown in Figure~\ref{fig:beam_csqmi_speed}. We find that Approx-FSMI is $1.7$ times faster than Approx-CSQMI, and Uniform-FSMI is $3$ times faster than Approx-CSQMI. The acceleration is not as large as predicted in Theorem~\ref{theorem:num_mul} due to compiler optimizations applied to Approx-CSQMI. 

\subsection{Simulated Scenario for 2D Mapping Algorithms}
\label{sec:experiments:fsmi-simulated}

In this section, we consider mapping in a synthetic environment, shown in Figure~\ref{fig:syn_explore_map}. 
The environment is $\SI{18}{\meter}\times \SI{18}{\meter}$, and it is represented by an occupancy grid with resolution $\SI{0.1}{\meter}$. 
The virtual robot that can measure range in all directions with $\SI{2}{\degree}$ resolution, thus emulating a laser range finder with 180 beams. 
We compare the following four algorithms: {\em (i)} the Approx-FSMI algorithm with truncation parameter set to $\Delta = 3$, {\em (ii)} the Approx-CSQMI algorithm with truncation parameter set to $\Delta = 3$, {\em (iii)} the Uniform-FSMI Algorithm, and {\em (iv)} the nearest frontier exploration method. 
In the first three cases, the robot chooses to follow paths, computed using Dijkstra's algorithm~\citep{cormen2009introduction}, that maximize the ratio between the mutual information gain along the path and the travel distance. In the fourth case, we use the nearest frontier algorithm discussed in~\cite{charrow2015csqmi}, where the robot travels to the closest cluster of frontier cells.
In all cases, the exploration terminates when the reduction of the entropy of the map representation falls below a constant threshold.  All algorithms are run three times and the results are averaged.

The results are shown in Figure~\ref{fig:2d_syn_environment_vis}. 
An example map generated by the Approx-FSMI algorithm as well as the path spanned by the robot are shown in Figure~\ref{fig:syn_explore_map}. The average path lengths spanned by the four exploration strategies are shown in Figure~\ref{fig:syn_explore_length}. We find that the first three exploration algorithms, all based on information-theoretic metrics, execute paths with approximately the same length. The frontier exploration method executes paths that are on average at 22\% longer than the information-based methods.

\begin{figure}[!b]
\centering 
\subfigure[Map and trajectory]{\label{fig:syn_explore_map}\includegraphics[height=0.4\columnwidth]{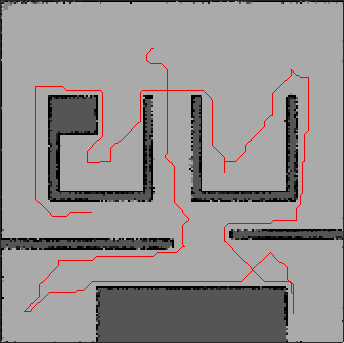}}
\subfigure[Trajectory length]{\label{fig:syn_explore_length}\includegraphics[height=0.4\columnwidth]{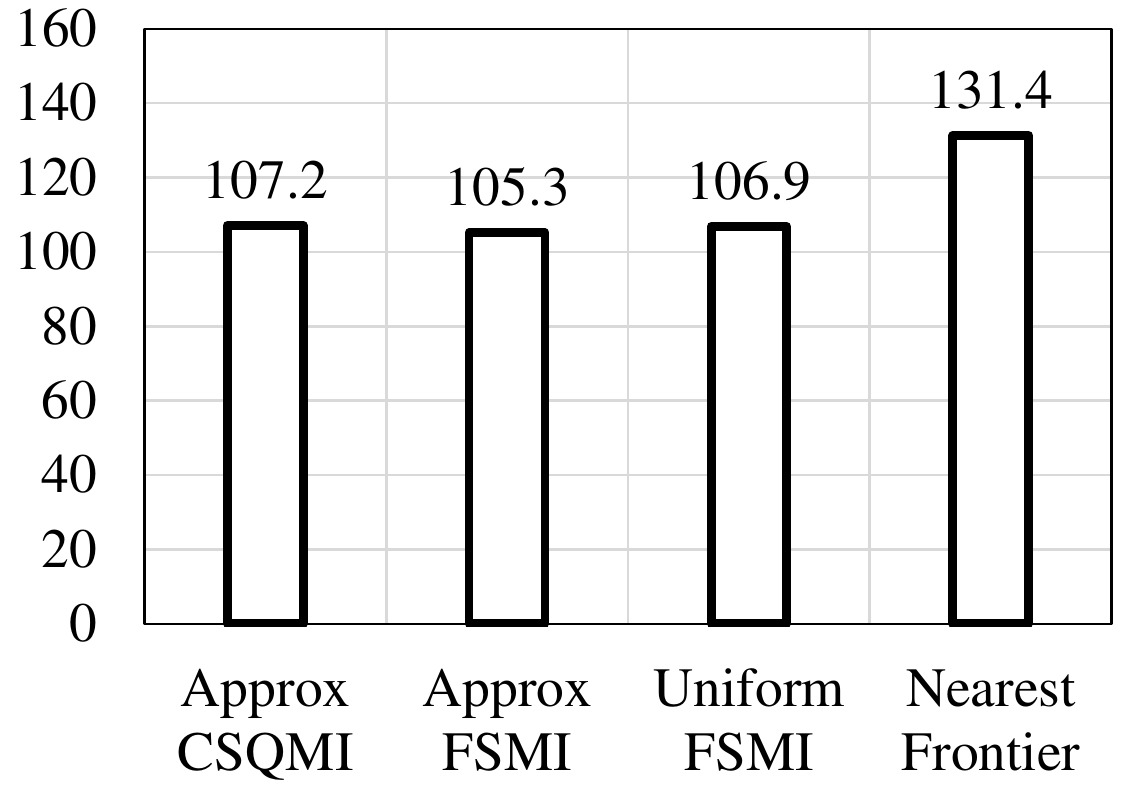}}
\caption{Synthetic 2D environment exploration experiment results.}
\label{fig:2d_syn_environment_vis}
\end{figure}

Throughout the experiments, the average time spent computing mutual information was recorded. The Approx-CSQMI algorithm takes $\SI{12.4}{\micro\second}$ per beam, the Approx-FSMI algorithm takes $\SI{8.3}{\micro\second}$ per beam, and the Uniform-FSMI takes $\SI{4.9}{\micro\second}$ per beam, on average. The ranking of the algorithms is consistent with our results reported in the Section~\ref{sec:experiments:fsmi-single}. In these experiments, the average beam length is roughly half of what was used in the Section~\ref{sec:experiments:fsmi-single}, hence these evaluation times are roughly half as well.

\subsection{Real-world Scenario for 2D Mapping Algorithms}
\label{sec:plannar_mapping_exp}

In this section, we consider a mapping scenario involving a 1/10th-scale car-like robot equipped with a Hokoyo UST-10LX LiDAR. In our experiments, we limited the field of view of the LiDAR to $\SI{230}{\degree}$ and its range to $\SI{3}{\meter}$. The sensor provides 920 range measurements within this field of view. 
The robot is placed in the $\SI{8}{\meter}$-by-$\SI{8}{\meter}$ environment shown in Figure~\ref{fig:real_car}. The location of the robot is obtained in real time using an OptiTrack motion capture system. Path planning is accomplished using the RRT$^*$ algorithm~\citep{rrtstar}, with Reeds-Shepp curves~\citep{reedsshepp} as the steering function. In our experiments, we evaluate mutual information along each path in the RRT$^*$ tree at $\SI{0.2}{\meter}$ intervals with 50 equally-spaced beams. We find the path that maximizes the ratio between the mutual information along the path and the total path length. A visualization of these potential paths and their rankings is shown in Figure~\ref{fig:exp_real_motion}. Once the car reaches the end of the selected trajectory, it computes a new trajectory using the same algorithm. This procedure is repeated until the entropy of the map drops below a constant threshold. In all of these experiments, the resolution of the occupancy grid map is set to $\SI{0.05}{\meter}$, and the sensor parameters are set to $\sigma = \SI{0.05}{\meter}$ and $\delta_{occ}=1/\delta_{emp}=1.5$.

In this setting, we compare the following three algorithms: {\em (i)} the Approx-CSQMI Algorithm with truncation parameter set to $\Delta = 3$, {\em (ii)} the Approx-FSMI Algorithm with truncation parameter set to $\Delta = 3$, {\em (iii)} the Uniform-FSMI algorithm. 
We found that all three algorithms perform similarly in terms of how quickly they reduce the entropy of the map as shown in Figure~\ref{fig:exp_real_entropy}. 
We measured Approx-CSQMI to take $\SI{422.7}{\micro\second}$, Approx-FSMI to take $\SI{148.7}{\micro\second}$, and Uniform-FSMI to take $\SI{111.4}{\micro\second}$ per beam, on average. 
The ranking of these timings is consistent with Sections~\ref{sec:experiments:fsmi-single} and~\ref{sec:experiments:fsmi-simulated}. Differences in scaling may be due to the differences in compiler optimizations on the ARM Cortex-A57 CPU present in the NVIDIA Tegra X2 platform.
We also provide a timelapse of the exploration with the Approx-FSMI algorithm in Figure~\ref{fig:2d_storyboard}. 

\begin{figure}[!t]
\centering 
\subfigure[Paths computed with the RRT* based planner are colored based on their potential information gain. Paths that obtain a the most mutual information per distance are colored green.]{\label{fig:exp_real_motion}\includegraphics[trim={9cm 5cm 3cm 7.5cm},clip,width=0.8\columnwidth]{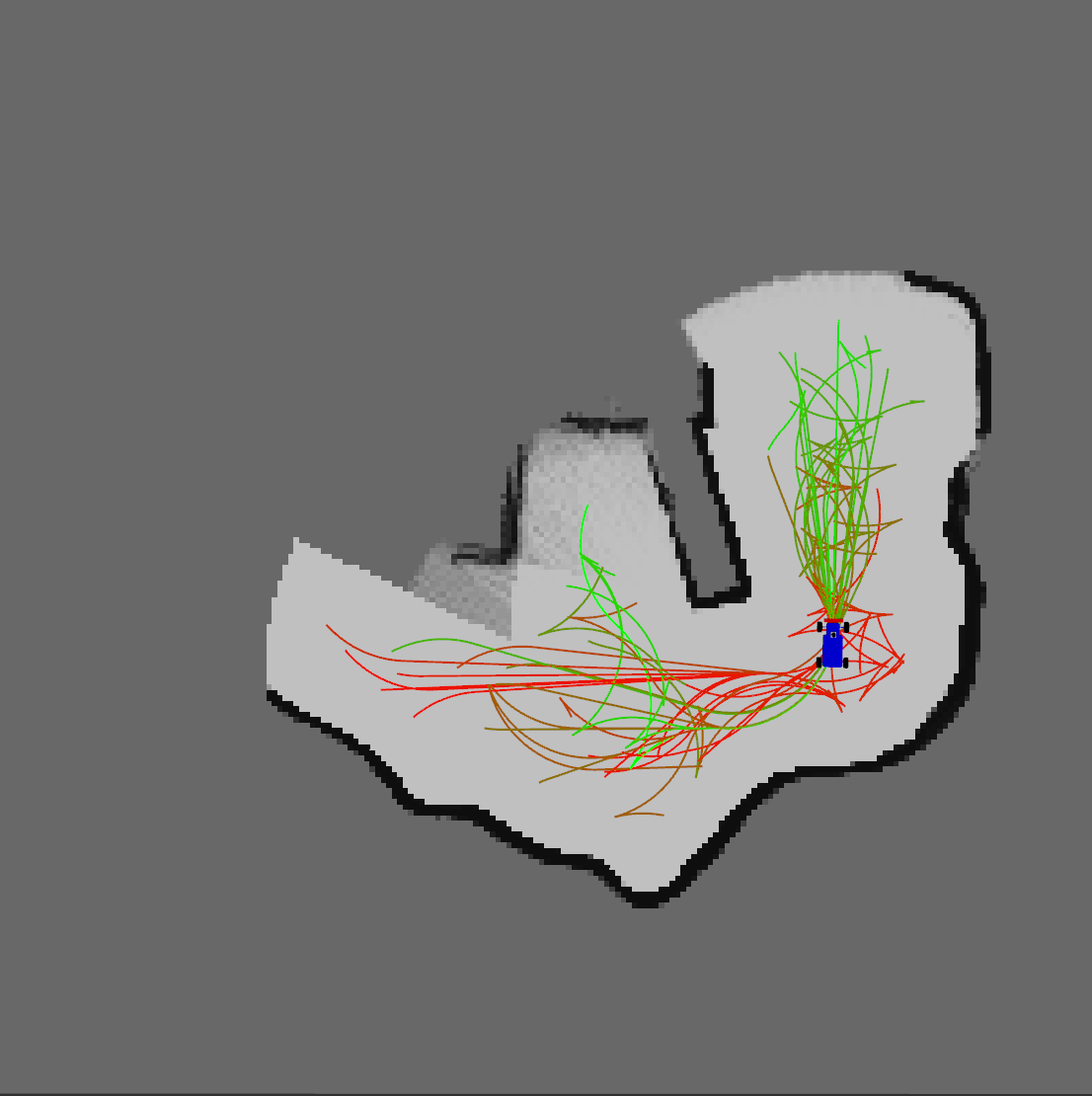}}

\subfigure[Representative samples of the map's entropy over the course of exploration.]{
\label{fig:exp_real_entropy}\includegraphics[trim={0 0 0 0},clip,width=0.95\columnwidth]{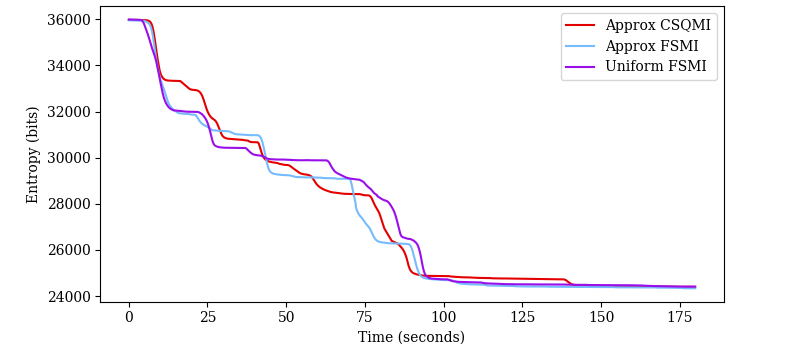}}
\caption{Real experiments with a car in a 2D environment.}
\end{figure}

\begin{figure*}
\centering
\begin{tabular}{cc}
\subfigure[The car begins exploration, revealing explorable areas on either side of the map.]{
\includegraphics[width=0.45\columnwidth]{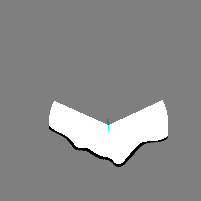}
\includegraphics[width=0.45\columnwidth]{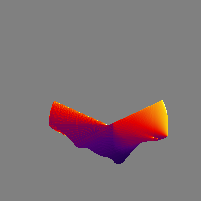}
}&
\subfigure[The car moves towards the larger of the two frontiers.]{
\includegraphics[width=0.45\columnwidth]{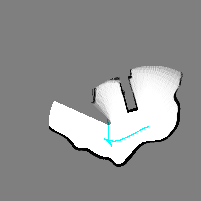}
\includegraphics[width=0.45\columnwidth]{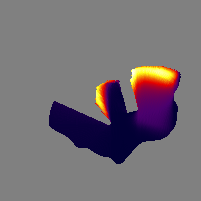}
}
\\
\subfigure[The car reaches the end of this branch with two small frontiers occluded by its field of view.]{
\includegraphics[width=0.45\columnwidth]{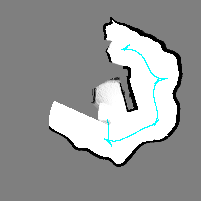}
\includegraphics[width=0.45\columnwidth]{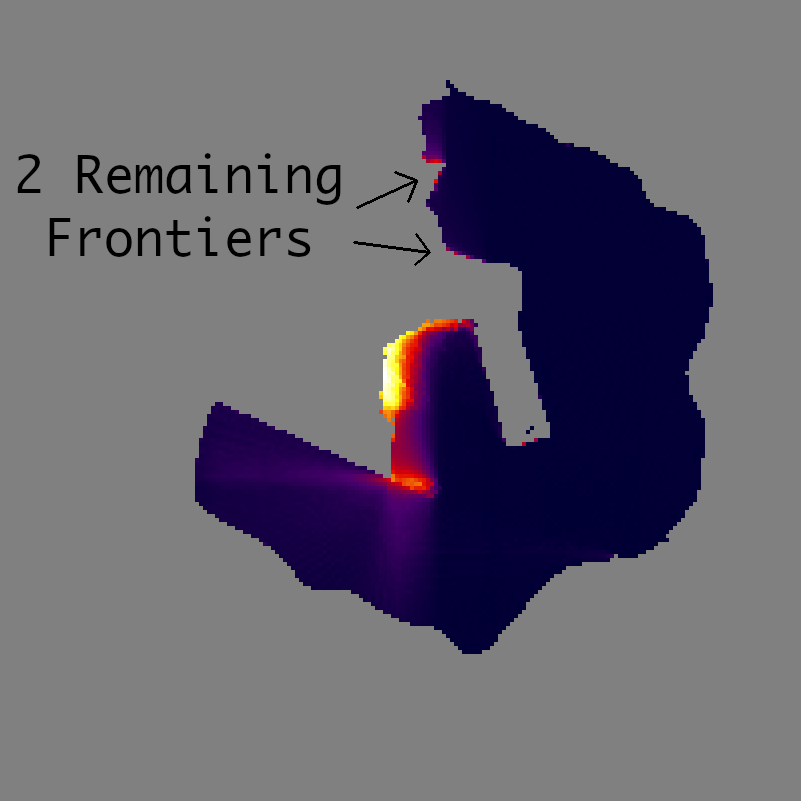}
}&
\subfigure[The car turns around to see an occluded corner, completing the right side of the map.]{
\includegraphics[width=0.45\columnwidth]{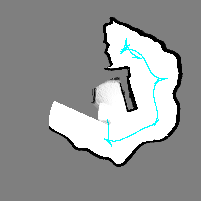}
\includegraphics[width=0.45\columnwidth]{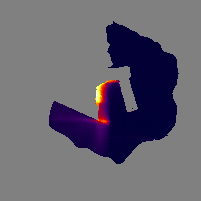}
}
\\
\subfigure[The car turns around to reveal the small inlet and the other side of the map.]{
\includegraphics[width=0.45\columnwidth]{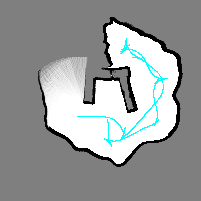}
\includegraphics[width=0.45\columnwidth]{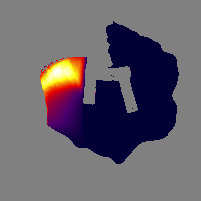}
}&
\subfigure[The car reaches the end of the branch, again revealing all but two frontiers.]{
\includegraphics[width=0.45\columnwidth]{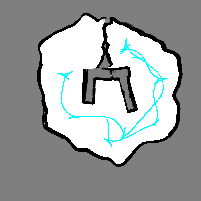}
\includegraphics[width=0.45\columnwidth]{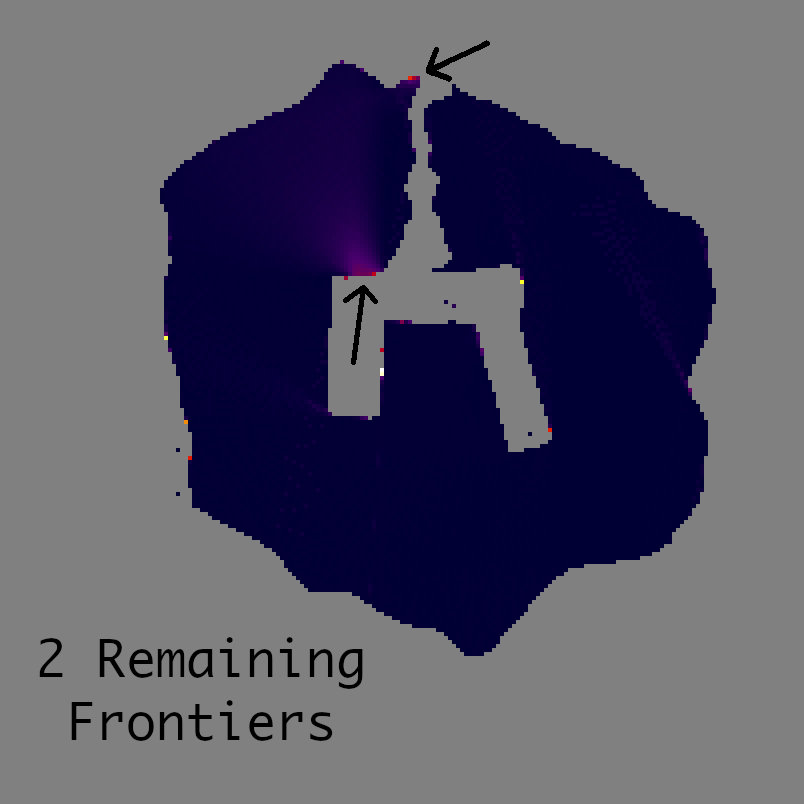}
}
\\
\subfigure[The car maneuvers to explore the last two corners, competing the map.]{
\includegraphics[width=0.45\columnwidth]{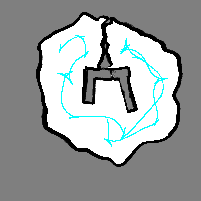}
\includegraphics[width=0.45\columnwidth]{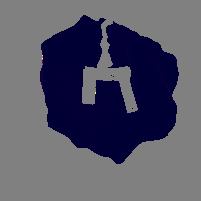}
}&
\subfigure[The real world environment and the car in its starting position.]{
\includegraphics[width=0.95\columnwidth]{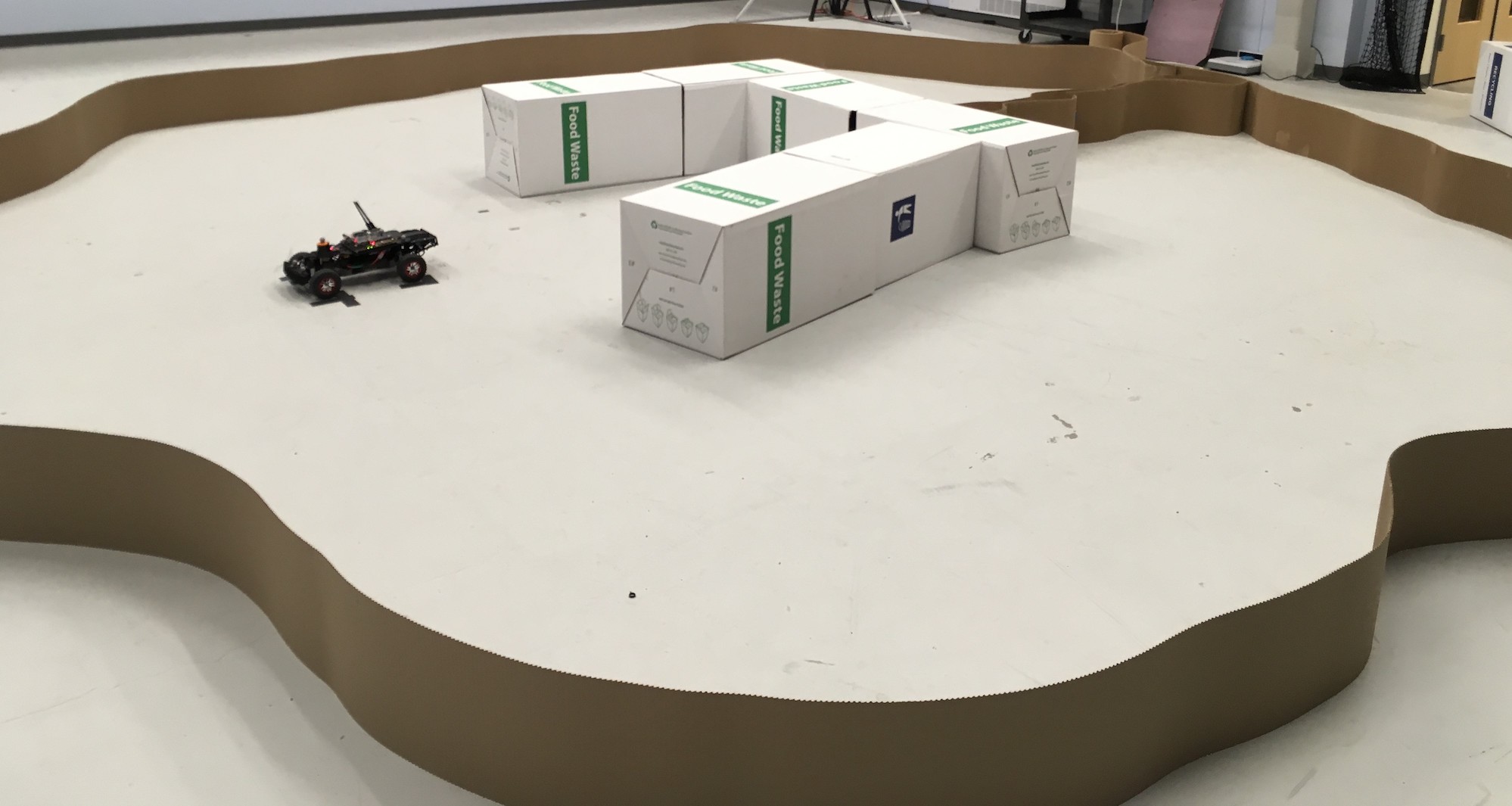}
\label{fig:real_car}
}

\end{tabular}
\caption{
A timelapse of a 2D exploration experiment. On the left, the black and white figures are occupancy maps where the dark pixels are the most likely to be occupied.
The blue line shows the path taken by the car.
On the right the heat map figures are the corresponding mutual information surfaces with bright colors representing scanning locations with high mutual information.
Accompanying video: \url{https://youtu.be/6Ia0conjKMQ}
}
\label{fig:2d_storyboard}
\end{figure*}

\begin{figure*}[t]
\centering
\begin{tabular}{cc}
\subfigure[The car begins exploration, revealing areas with high MI in the top and bottom of the map.]{
\includegraphics[width=0.45\columnwidth,trim={570 90 720 80},clip]{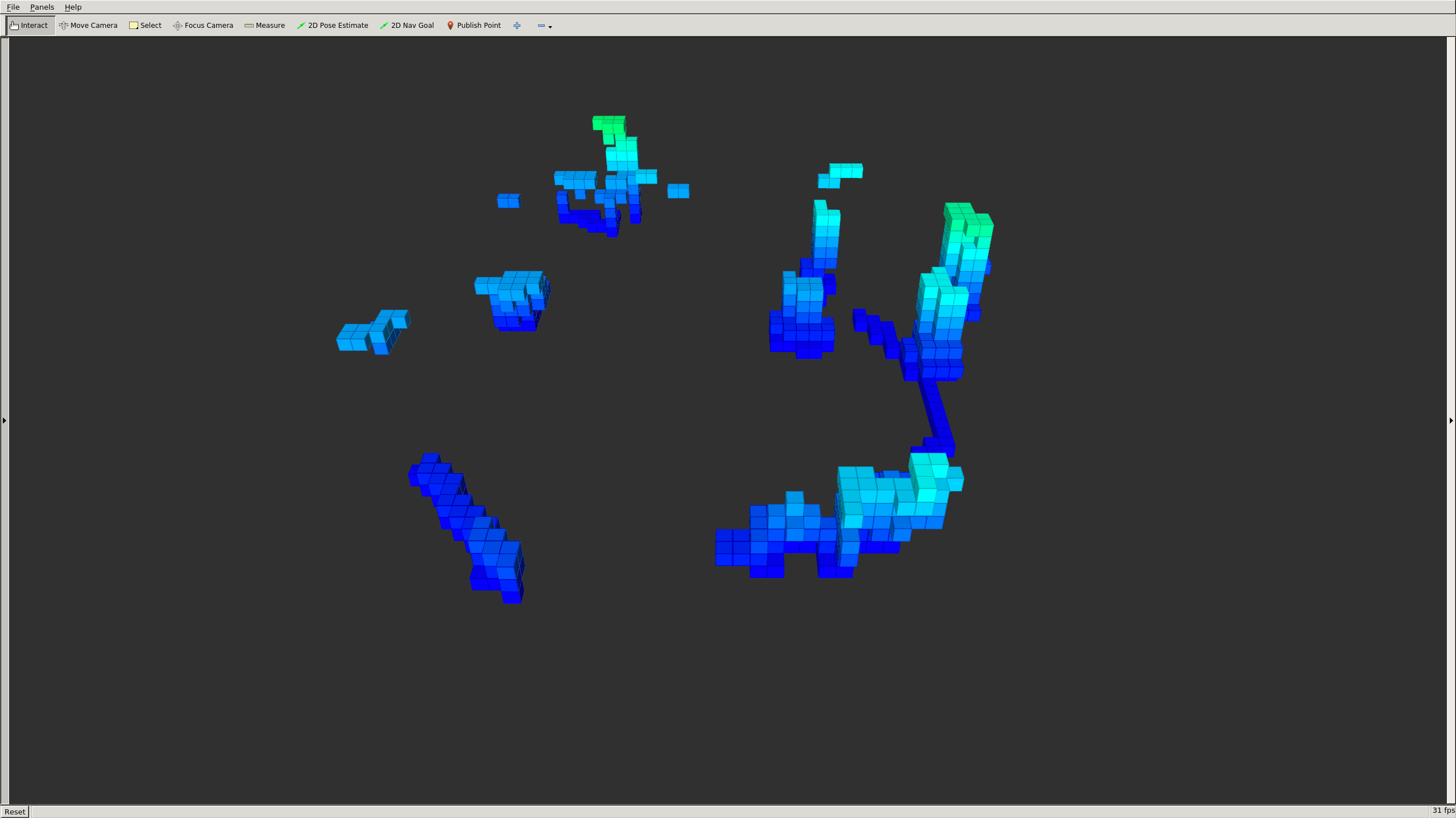}
\includegraphics[width=0.45\columnwidth,angle=90,trim={57 57 90 90},clip]{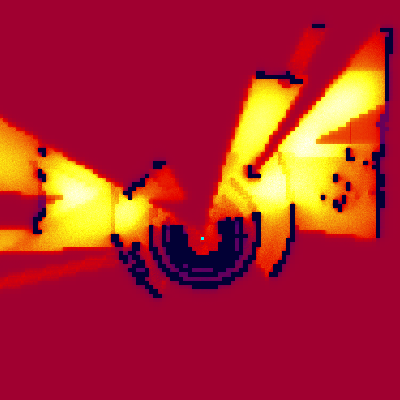}
}&
\subfigure[The car moves to the top left revealing a large amount of free space (not visible in the 3D map) as well as the top of the tree in the bottom right. The regions in the lower half of the map now have the highest MI.]{
\includegraphics[width=0.45\columnwidth,trim={570 90 720 80},clip]{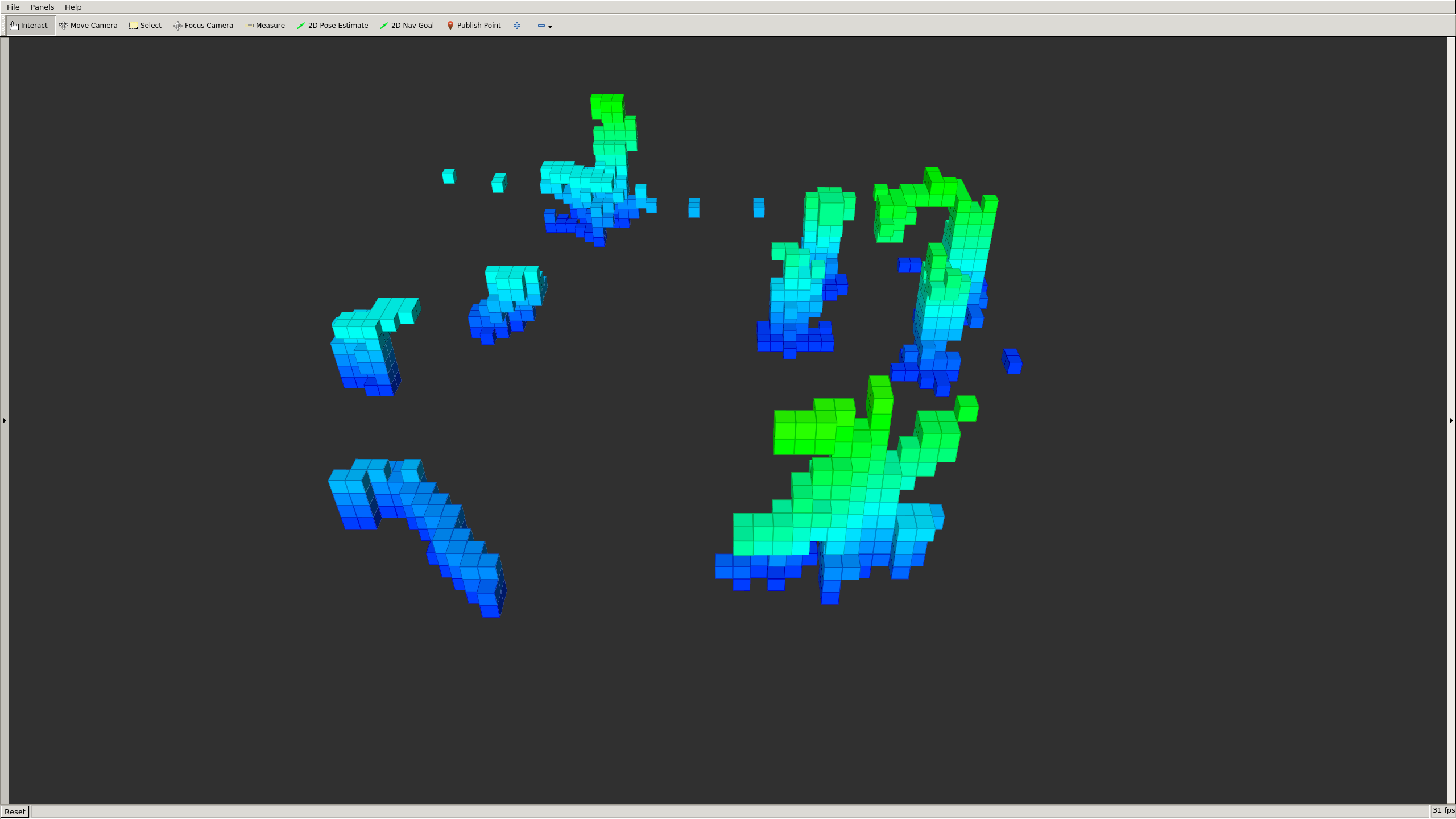}
\includegraphics[width=0.45\columnwidth,angle=90,trim={57 57 90 90},clip]{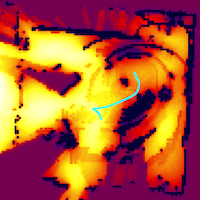}
}
\\
\subfigure[The car moves to the bottom of the map revealing the top of the arch in the top left and the backs of the box in the bottom left and tree in the bottom right.]{
\includegraphics[width=0.45\columnwidth,trim={570 90 720 80},clip]{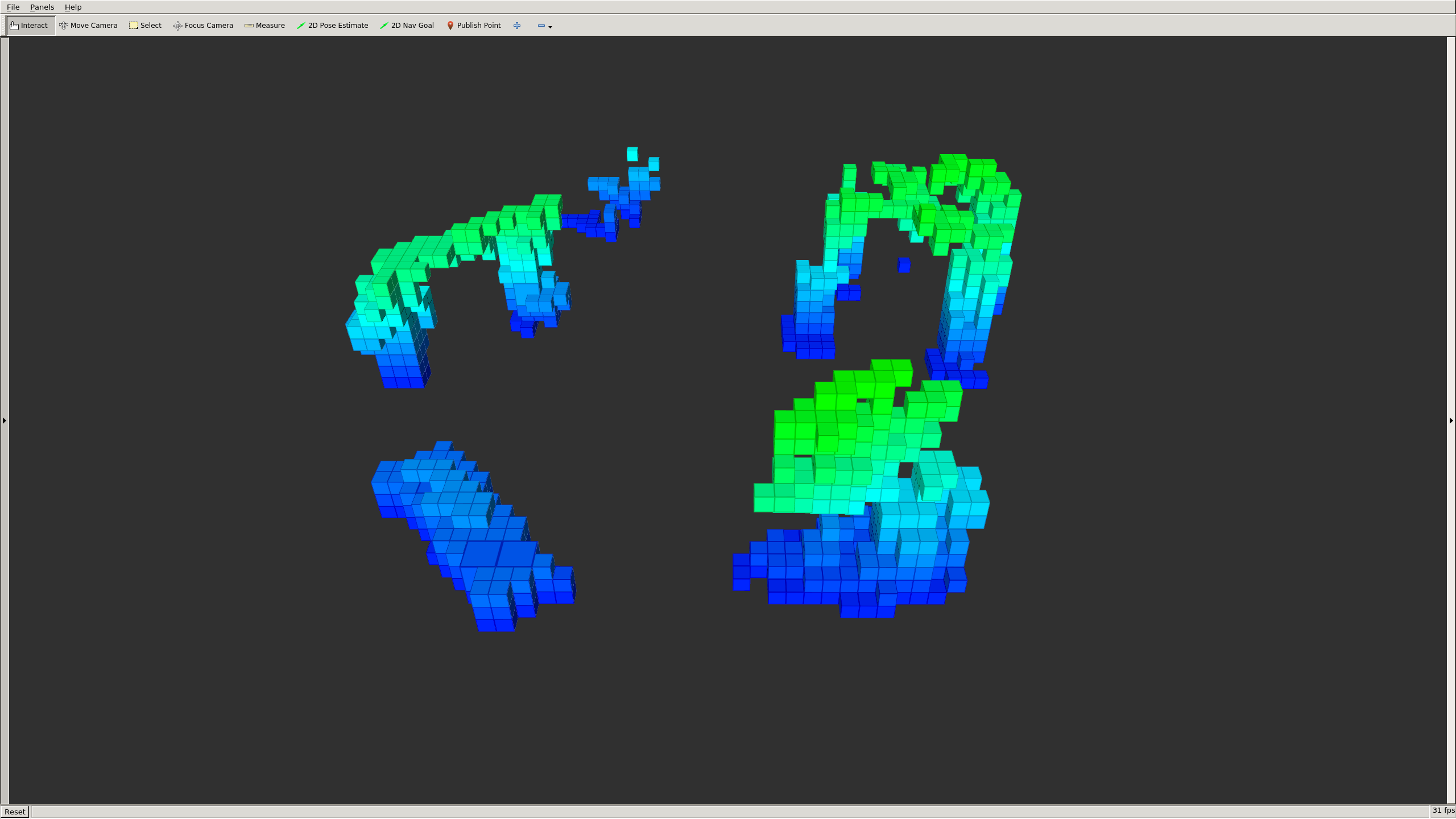}
\includegraphics[width=0.45\columnwidth,angle=90,trim={57 57 90 90},clip]{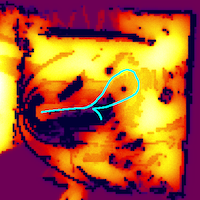}
}&
\subfigure[The car moves to a point in the far top left with, perhaps because it will reveal the top of the cat in the top right. In doing so, the arch occludes it from the motion capture cameras, distorting the map. The break in the trajectory is visible in the next frame.]{
\includegraphics[width=0.45\columnwidth,trim={570 90 720 80},clip]{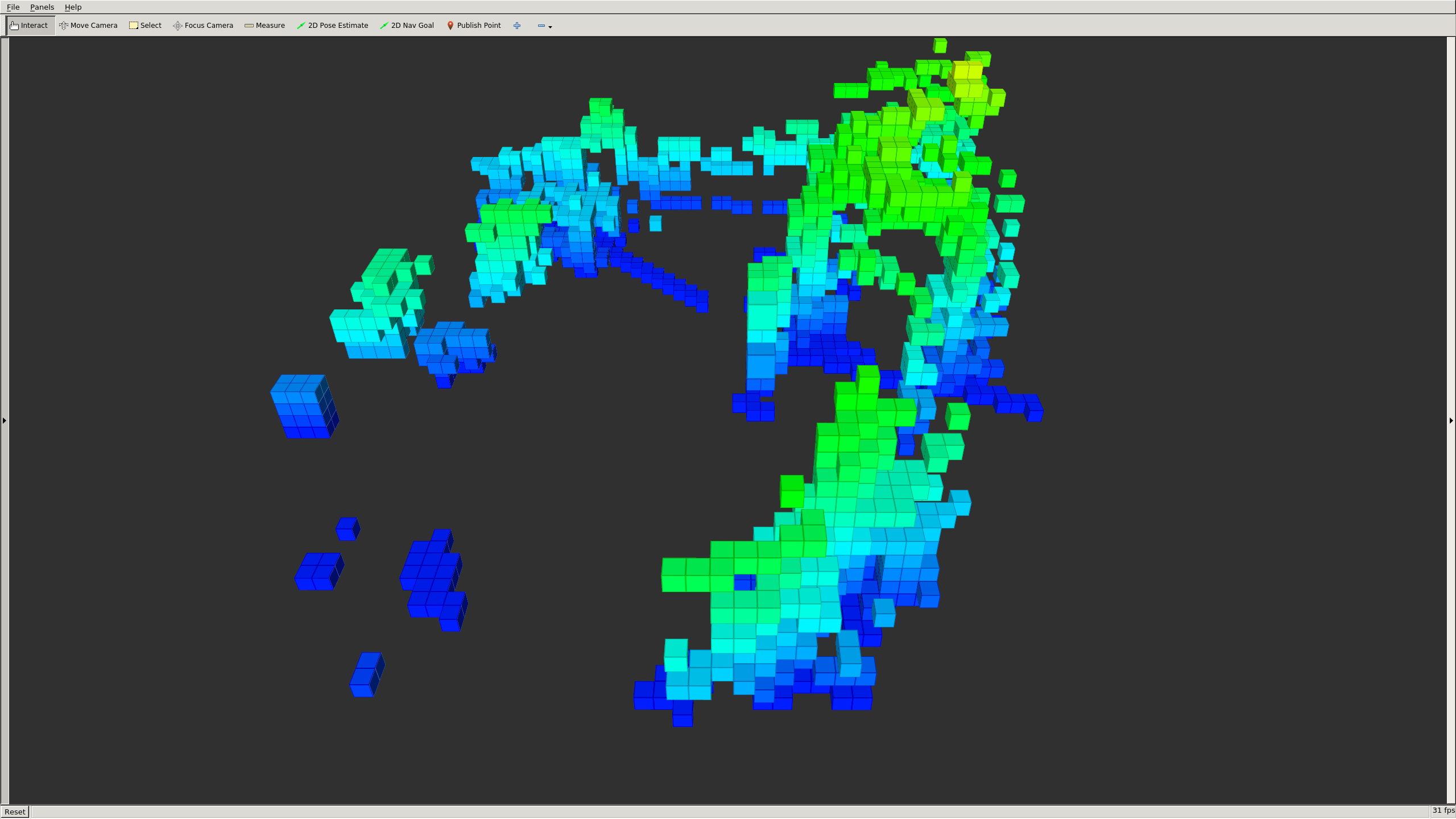}
\includegraphics[width=0.45\columnwidth,angle=90,trim={57 57 90 90},clip]{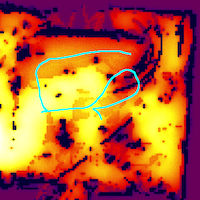}
}
\\
\subfigure[The car moves to the bottom right and recovers from most of the distortion.]{
\includegraphics[width=0.45\columnwidth,trim={570 90 720 80},clip]{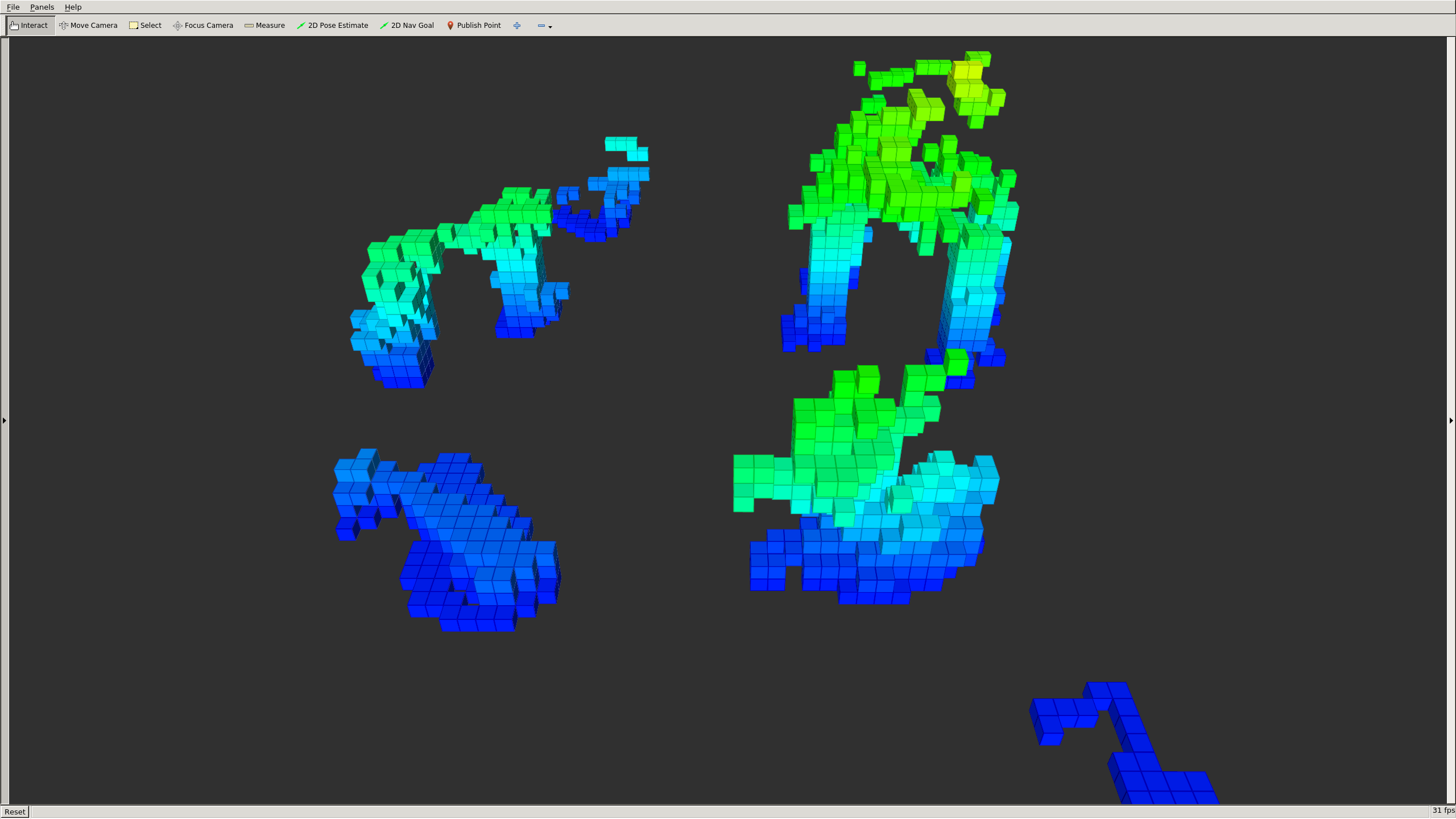}
\includegraphics[width=0.45\columnwidth,angle=90,trim={57 57 90 90},clip]{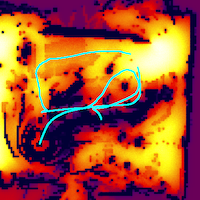}
}&
\subfigure[The car moves to the top left and recovers some information about the back of the arch and cat (not visible in these figures).]{
\includegraphics[width=0.45\columnwidth,trim={570 90 720 80},clip]{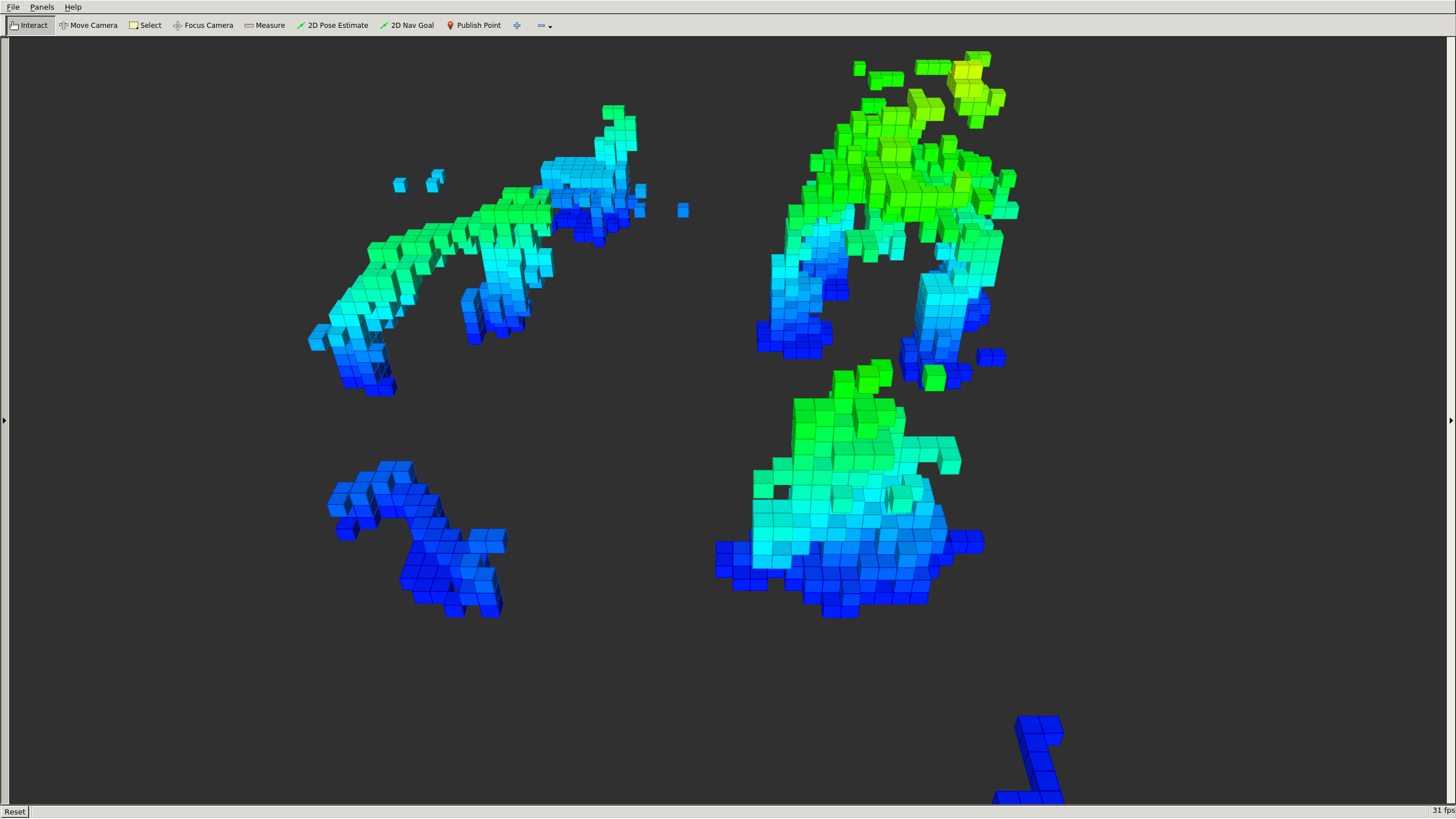}
\includegraphics[width=0.45\columnwidth,angle=90,trim={57 57 90 90},clip]{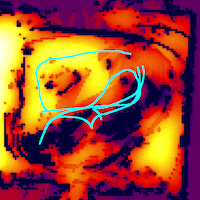}
}
\\
\subfigure[The car moves to the bottom left, improving the top of the cat as well as the box.]{
\includegraphics[width=0.45\columnwidth,trim={570 90 720 80},clip]{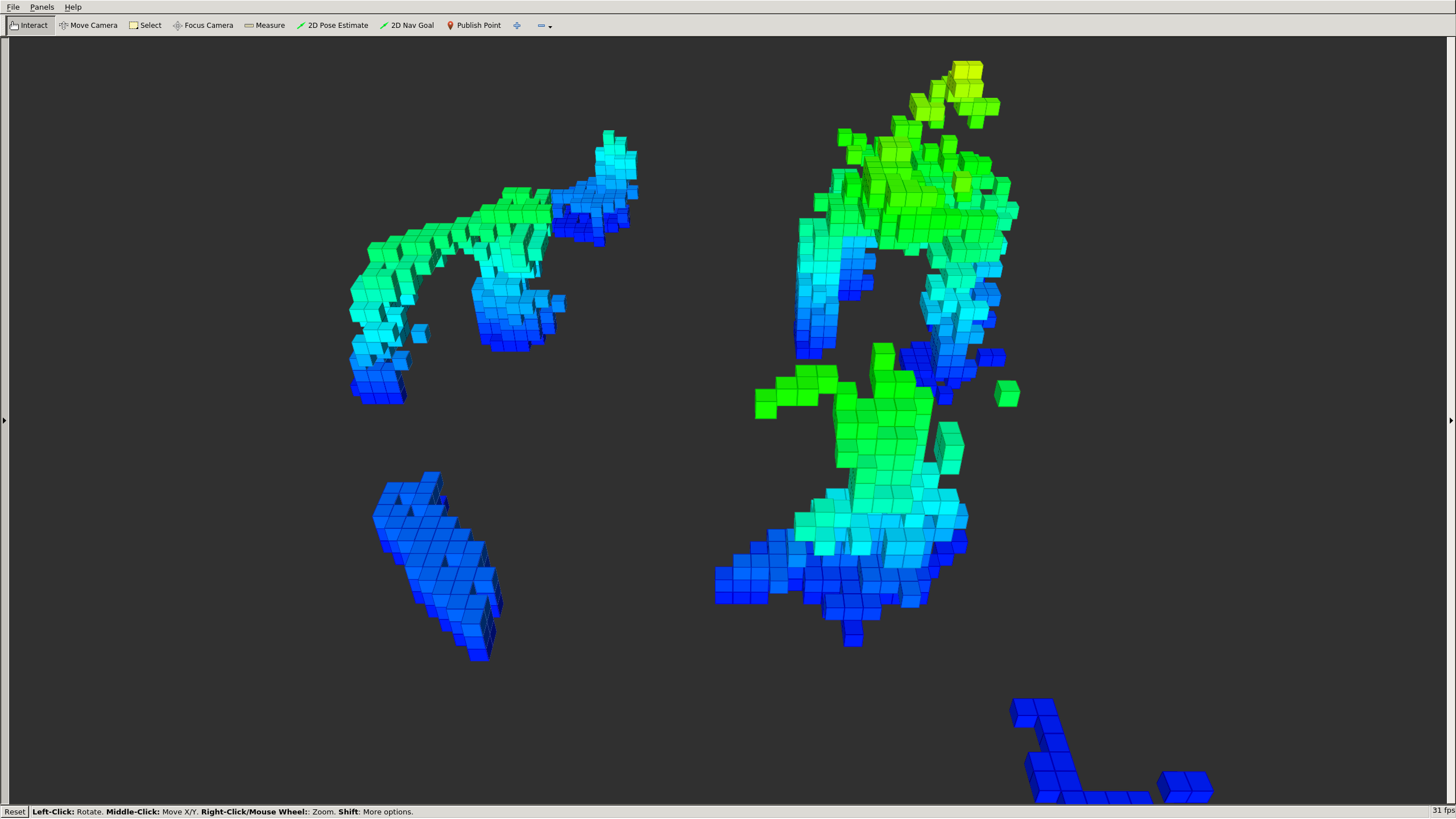}
\includegraphics[width=0.45\columnwidth,angle=90,trim={57 57 90 90},clip]{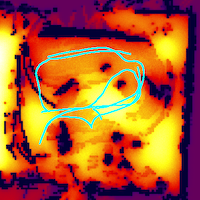}
}&
\subfigure[The environment featuring an arch, a giant cat, a box and a tree shown from the same perspective.]{
\includegraphics[width=0.67\columnwidth]{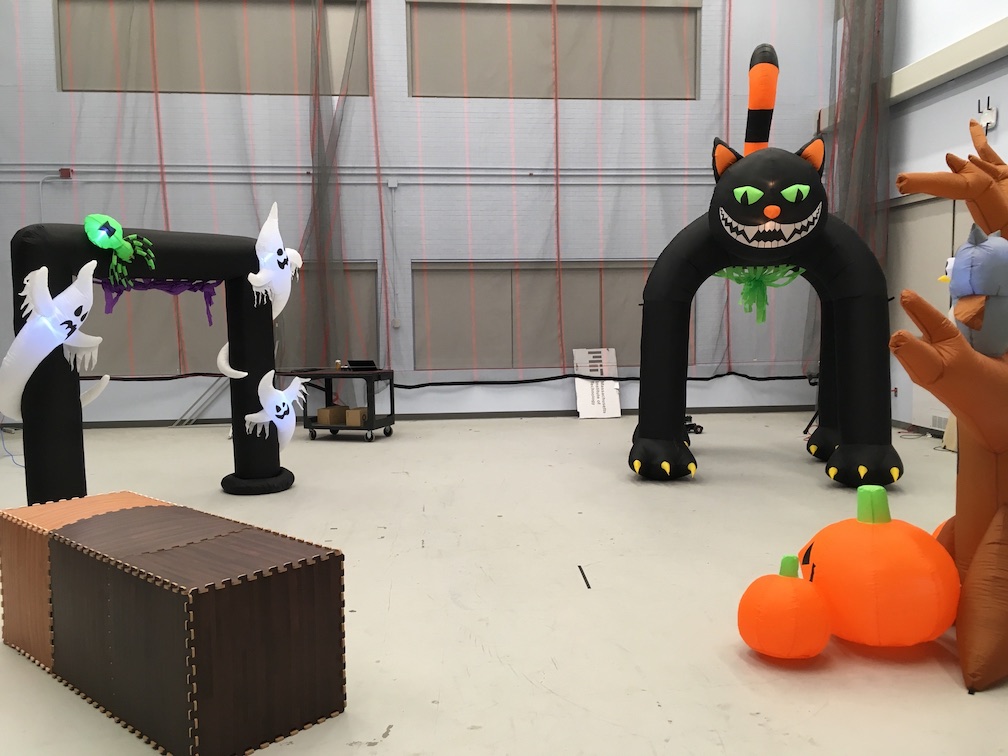}
\label{fig:real_car_3D}
}
\end{tabular}
\caption{
A timelapse of a 3D exploration experiment in the environment pictured in (h).
The tallest obstacle, the cat, is $\SI{4}{\meter}$ tall, $20\times$ taller than the car.
On the left of each figure is a 3D OctoMap where color indicates the height of a voxel.
On the right, is a 2D projection of the corresponding mutual information surface.
For each point in the image, we compute the mutual information from a scan consisting of beams are randomly sampled from the LiDAR's $360^\circ$ horizontal field of view and $\pm 15^\circ$ vertical field of view.
Localization is done using motion capture cameras as in the 2D experiments; however, this is occasionally lost in the case where the 3D obstacles occlude the camera view.
}
\label{fig:3d_storyboard}
\end{figure*}

\subsection{Computational Experiments for the Approx-FSMI-RLE Algorithm} \label{sec:experiments:approx-fsmi-single}

In this section, we evaluate the Approx-FSMI-RLE algorithm which calculates the Shannon mutual information of occupancy sequences represented using the run-length encoding. We compare the following two algorithms in terms of run time: {\em (i)} the Approx-FSMI-RLE algorithm with truncation parameter set to $\Delta=3$ executing on the compressed measurement using run-length-encoding, {\em (ii)} the Approx-FSMI algorithm with truncation parameter set to $\Delta = 3$ executing on the uncompressed measurement.

We consider a synthetic beam passing through $n=256$ occupancy cells divided into groups of $L$ cells. Within each group, the cells are given the same occupancy values. Run-length encoding compresses this sequence by a factor of $L$, so $L$ can be interpreted as the compression ratio. We run each algorithm 1000 times with randomly-generated occupancy values. 

The results are presented in Table~\ref{table:fsmi_rle_vs_fsmi_speed}. We find that the Approx-FSMI-RLE algorithm is slower than the Approx-FSMI algorithm when $L < 4$. However, it achieves an order of magnitude speedup for $L \ge 32$. To avoid the overhead of the Approx-FSMI-RLE algorithm when $L<4$, we can adaptively switch between using the Approx-FSMI algorithm on the decompressed sequence and using Approx-FSMI-RLE algorithm on compressed sequence depending on whether $L<4$ or not.

\begin{table}[!tb]
\centering
\caption{
    \textup{
        The run time of the Approx-FSMI-RLE algorithm vs that of the FSMI-RLE algorithm (\SI{}{\micro\second}). The baseline algorithm, the Approx-FSMI algorithm that operates on the uncompressed occupancy vector, takes $56.1$\SI{}{\micro\second} to complete.
    }
}
\label{table:fsmi_rle_vs_fsmi_speed}
\resizebox{\columnwidth}{!}{
\begin{tabular}{|c|c|c|c|c|c|c|c|c|}
 \hline
 & $L = 1$ & $L = 2$ & $L = 4$ & $L = 8$ & $L = 16$ & $L = 32$ & $L = 64$ & $L = 128$\\\hline
 \makecell{Run time of \\Approx-FSMI-RLE} & 240.9 & 79.4 & 31.5 & 12.3 & 7.6 & 4.9 & 3.4 & 2.3\\\hline
\makecell{Run time ratio of\\
the two algorithms} & 0.2 & 0.7 & 1.8 & 4.6 & 7.4 & 11.2 & 16.5 & 24.4\\\hline
\end{tabular}
}
\end{table}

\subsection{Real-world Scenario for the Approx-FSMI-RLE Algorithm}\label{sec:experiments:approx-fsmi-real-world}

In this section, we evaluate the Approx-FSMI-RLE algorithm using the same 1/10-scale car-like robot described in Section~\ref{sec:plannar_mapping_exp}.
In these experiments, the car is now equipped with a Velodyne Puck LiDAR with 16 channels with a vertical field of view of $\pm 15^\circ$. The horizontal field of view is limited to $270^\circ$ due to occlusions on the frame of the car.
Although the robot travels on the 2D plane, this range sensor allows it to measure the environment in three dimensions.

We use the Approx-FSMI-RLE algorithm to evaluate mutual information in a three-dimensional mapping task.
We use the OctoMap software package~\citep{octo} to represent the map, which naturally returns run-length encoded occupancy sequences.
Once again we obtain location information from the OptiTrack motion capture system. 
We utilize the RRT$^*$ algorithm as described in Section~\ref{sec:plannar_mapping_exp} to generate a set of possible trajectories and then choose to travel along the one that maximizes the ratio between the mutual information along the path and the length of the path. 
We evaluate the mutual information along each path at $\SI{0.2}{\meter}$ intervals using the Approx-FSMI-RLE algorithm. For this purpose, we randomly sample 100 beams within the LiDAR's $270^\circ$ horizontal field of view and $\pm 15^\circ$ vertical field of view. 

An example experiment including the resulting map, the mutual information surface, and path are shown in Figure~\ref{fig:3d_storyboard} at various intervals. As compared to Section~\ref{sec:plannar_mapping_exp}, where the mutual information gain tended to be highest near boundaries, we observe that mutual information in these three-dimensional experiments is higher in the center of unknown regions. We conclude that due to the car's limited vertical vision this phenomenon exists to enforce the car to ``take a step back'' in order to get a better view.

We repeated the same experiments using the Approx-FSMI algorithm. We observed that, on average, the run-length encoding compresses occupancy sequences by roughly $18$ times. We found that this compression enables the Approx-FSMI-RLE algorithm to be $8$ times faster than the Approx-FSMI in terms of run time, even when we exclude the time to decompress the RLE representation coming from the OctoMap to the uncompressed representation used by the Approx-FSMI algorithm.


\section{Conclusion}
In this paper, we introduced Fast Shannon Mutual Information (FSMI), an algorithm for computing the Shannon mutual information between future measurements and an occupancy grid map. 
For 2D information theoretic mapping scenarios, we introduced three algorithms: FSMI, Approx-FSMI which approximates FSMI with arbitrary precision, and Uniform-FSMI which computes exact Shannon mutual information under the assumption that the measurement noise is uniformly distributed. 
We also extend the algorithms to 3D mapping tasks when an OctoMap data structure is used to represent the map. We introduced the FSMI-RLE algorithm that accelerates the FSMI algorithm with a run-length encoding (RLE) compression technique. To address numerical issues inherent to the FSMI-RLE algorithm, we proposed the Approx-FSMI-RLE algorithm which utilizes Gaussian truncation. We also discussed the Uniform-FSMI-RLE algorithm which parallels the Uniform-FSMI algorithm.

We have rigorously proved guarantees on the correctness and the computational complexity of the proposed algorithms. 
In our computational experiments, we showed that the FSMI algorithm achieved more than three orders of magnitude computational savings when compared to the original Shannon mutual information computation algorithm described in~\citep{julian2014mutual}, while maintaining higher accuracy. 
We showed that the Approx-FSMI runs $7$ times faster than FSMI with negligible precision loss. 
We then showed that the Approx-FSMI algorithm has the same asymptotic computational complexity as the Approx-CSQMI algorithm. We also showed that the Approx-FSMI algorithm has a smaller constant factor than the Approx-CSQMI algorithm, which is measured by the number of multiplication operations, 
In our computational experiments, we observed that indeed the FSMI algorithm runs twice as fast as the Approx-CSQMI algorithm. Moreover, we showed that if the measurement noise follows a uniform distribution, the Shannon mutual information can be computed \emph{exactly} in linear time with respect to occupancy resolution by the Uniform-FSMI algorithm. 
We also conducted synthetic and real-world experiments to demonstrate the performance of the Approx-FSMI-RLE algorithm over the Approx-FSMI. Our experiments show $8$ times of acceleration in practice.

\bibliographystyle{SageH}
\bibliography{egbib}

\appendix
\section{Analytical Solution to the Summation}
\label{sec:appendix_analytic}
This section derives the closed-form solution to the following two terms\begin{equation}
\begin{split}
\label{eqn:AB_original}
A[x, L_u, L_v, t] &= \sum_{j=0}^{L_u - 1}\sum_{k=0}^{L_v - 1} x^j \exp{\left(-\frac{(j-k+t)^2}{2\sigma^2}\right)}\\
B[x, L_u, L_v, t] &= \sum_{j=0}^{L_u - 1}\sum_{k=0}^{L_v - 1} k\cdot x^j \exp{\left(-\frac{-(j-k+t)^2}{2\sigma^2}\right)}.
\end{split}
\end{equation}
As stated before, we have not found an analytic way to evaluate $A[x, L_u, L_v, t]$ and $B[x, L_u, L_v, t]$. Fortunately, if we allow for slight approximation, the closed-form solutions to both terms show up.
Specifically, we approximate the summation $A[x, L_u, L_v, t] = \sum_{j=0}^{L_u - 1}\sum_{k=0}^{L_v - 1} x^{j} \exp{\left(-\frac{(j-k+t)^2}{2\sigma^2}\right)}$ by integration:
\begin{equation}
\label{eqn:approx_xi_exp}
\int_{-\frac{1}{2}}^{L_u - \frac{1}{2}} \int_{-\frac{1}{2}}^{L_v - \frac{1}{2}}  x^j \exp{\left(-\frac{\left(j-k+t\right)^2}{2\sigma^2}\right)} \mathrm{d}j \mathrm{d}k.
\end{equation}

To derive the closed-form solution to this integration, we need the following lemma:
\begin{lemma}
Let $\erfc(\cdot)$ be the complementary error function. Let $0 < x < 1$. We have
\begin{equation}
\begin{split}
    &\theta_{\sigma, x}(a_1, a_2, t) = \int_{a_1}^{\infty} \int_{a_2}^{\infty} x^{j} \exp{\left(-\frac{\left(j - k + t\right)^2}{2\sigma^2}\right)} \mathrm{d}j \mathrm{d}k = \\
    &\frac{1}{\log{x}} \Bigg(\sqrt{\frac{\pi}{2}} x^{-t}\sigma'\Bigg(x^{a_1 + t}\left(-2 + \erfc\left(\frac{a_1 - a_2 + t}{\sqrt{2}\sigma}\right)\right) - \\
    & \mathrm{e}^{\frac{1}{2} \sigma^2\cdot  \log^2{x}} x^{a_2} \erfc\left(\frac{a_1 - a_2 + t-\sigma^2 \log^2{\bar{o}_u}}{\sqrt{2}\sigma}\right)\Bigg)\Bigg) \Bigg).
\end{split}
\end{equation}
In addition, $\theta_{\sigma, x}(a_1, a_2, t)$ can be evaluated in $O(1)$.
\end{lemma}

Based on the lemma, the analytic solution is:

\begin{theorem}[Closed-form solution to the approximation]
Eqn.~\eqref{eqn:approx_xi_exp} can be evaluated in $O(1)$ as follows:
\begin{equation}
\begin{split}
&\int_{-\frac{1}{2}}^{L_u - \frac{1}{2}} \int_{-\frac{1}{2}}^{L_v - \frac{1}{2}}x^{j} \exp{\left(-\frac{(j + t - k)^2}{2\sigma^2}\right)} \mathrm{d}j \mathrm{d}k = \\
&\theta_{\sigma,x}(-\frac{1}{2}, -\frac{1}{2}, t) - \theta_{\sigma, x}(L_u - \frac{1}{2}, -\frac{1}{2}, s_u - s_v) - \\
&\theta_{\sigma, x}(- \frac{1}{2}, L_v - \frac{1}{2}, t) + \theta_{\sigma, x}(L_u - \frac{1}{2}, L_v - \frac{1}{2}, t)
\end{split}
\end{equation}
\end{theorem}

Although this solution is closed-form and has $O(1)$ time complexity, evaluating it on the fly requires eight evaluation of the complementary error function and tens of expensive operations including exponential and square-root.

Similarly, we approximate $B[x, L_u, L_v, t] = \sum_{j=0}^{L_u - 1}\sum_{k=0}^{L_v - 1} k\cdot x^j \exp{\left(-\frac{\left(j - k + t\right)^2}{2\sigma^2}\right)}$ by 
\begin{equation}
\label{eqn:approx_j_xi_exp}
\int_{-\frac{1}{2}}^{L_u - \frac{1}{2}} \int_{-\frac{1}{2}}^{L_v - \frac{1}{2}} k\cdot x^j \exp{\left(-\frac{\left(j - k + t\right)^2}{2\sigma^2}\right)} \mathrm{d}j \mathrm{d}k.
\end{equation}

We can also show that Equation~\eqref{eqn:approx_j_xi_exp} has a closed-form solution which can be evaluated in $O(1)$. However, the solution is much more complicated than the solution for Equation~\eqref{eqn:approx_xi_exp} and we skip it in this paper. 
\end{document}